\documentclass{article}

 \usepackage[preprint]{neurips_2026}

\usepackage[utf8]{inputenc} 
\usepackage[T1]{fontenc}    
\usepackage{hyperref}       
\usepackage{url}            
\usepackage{booktabs}       
\usepackage{amsfonts}       
\usepackage{amsmath}
\usepackage{wrapfig} 
\usepackage{dsfont}
\usepackage{graphicx}
\usepackage{subcaption}
\usepackage{nicefrac}       
\usepackage{microtype}      
\usepackage{xcolor}         
\usepackage{algorithm}
\usepackage{algorithmic}
\usepackage{tabularx} 
\usepackage{amsthm}
\newtheorem{assumption}{Assumption}
\newtheorem{theorem}{Theorem}
\newtheorem{proposition}{Proposition}

\usepackage[most]{tcolorbox}
\usepackage{fvextra}

\usepackage{listings}

\definecolor{diffadd}{HTML}{1B6D2A}
\definecolor{diffdel}{HTML}{B0322B}
\definecolor{diffhunk}{HTML}{6F4FAA}
\definecolor{difffile}{HTML}{555555}
\definecolor{boxbg}{HTML}{FAFAF7}
\definecolor{boxframe}{HTML}{555555}
\definecolor{pycomment}{HTML}{777777}
\definecolor{pystring}{HTML}{0B5394}
\definecolor{pykw}{HTML}{8E2A6B}

\lstdefinelanguage{udiff}{
  morecomment=[f][\color{difffile}]{+++},
  morecomment=[f][\color{difffile}]{---},
  morecomment=[f][\color{diffhunk}\bfseries]{@@},
  morecomment=[f][\color{diffadd}]{+},
  morecomment=[f][\color{diffdel}]{-},
}

\lstdefinestyle{diffstyle}{
  language=udiff,
  basicstyle=\ttfamily\scriptsize,
  showstringspaces=false,
  breaklines=true,
  breakatwhitespace=false,
  postbreak=\mbox{\textcolor{gray}{\,$\hookrightarrow$\,}},
  columns=fullflexible,
  keepspaces=true,
  upquote=true,
  literate={─}{{-}}1 {–}{{-}}1 {—}{{-}}1,
  aboveskip=0.4em,
  belowskip=0.4em,
}

\lstdefinestyle{pystyle}{
  language=Python,
  basicstyle=\ttfamily\scriptsize,
  keywordstyle=\color{pykw}\bfseries,
  commentstyle=\color{pycomment}\itshape,
  stringstyle=\color{pystring},
  showstringspaces=false,
  breaklines=true,
  breakatwhitespace=false,
  postbreak=\mbox{\textcolor{gray}{\,$\hookrightarrow$\,}},
  columns=fullflexible,
  keepspaces=true,
  upquote=true,
  aboveskip=0.4em,
  belowskip=0.4em,
}

\newtcolorbox{codebox}[1]{%
  breakable,
  enhanced,
  colback=boxbg,
  colframe=boxframe,
  fonttitle=\bfseries\small,
  title={#1},
  boxrule=0.4pt,
  arc=2pt,
  left=4pt,
  right=4pt,
  top=3pt,
  bottom=3pt,
}

\definecolor{promptbg}{HTML}{F7F7F2}
\definecolor{promptframe}{HTML}{555555}
\definecolor{slotcolor}{HTML}{0B5394}

\newtcolorbox{promptbox}[1]{%
  breakable,
  enhanced,
  colback=promptbg,
  colframe=promptframe,
  fonttitle=\bfseries\small,
  title={#1},
  boxrule=0.5pt,
  arc=2pt,
  left=6pt, right=6pt, top=4pt, bottom=4pt,
}

\newcommand{\slot}[1]{\textcolor{slotcolor}{\texttt{\{#1\}}}}

\usepackage[dvipsnames]{xcolor}

\title{Self Improvement via Fast Tree-search}

\author{%
  Xinghong Fu\thanks{Correspondence to: fxh@mit.edu} \\
  Massachusetts Institute of Technology \\
  \And
  Aravinth Kulanthaivelu \\
  Sakana AI \\
  \And
  Yutaro Yamada \\
  Sakana AI \\
}

\begin{document}

\maketitle

\begin{abstract}
Coding agents can recursively modify their own implementations, forming a loop of self-improvement.
While prior work shows this can boost performance on coding benchmarks, existing approaches are costly and compute-intensive.
We introduce a simple, sample-efficient self-improvement framework that significantly improves coding performance under strict budget constraints.
We identify evaluation of candidate self-modifications as the main runtime bottleneck since prior approaches estimate their effectiveness by re-running a subset of benchmark tasks with the modified agent, which is time-consuming.
We introduce Recursive \textbf{S}elf \textbf{I}mprovement via \textbf{F}ast \textbf{T}ree-search (SIFT), which augments these downstream task evaluations with an LLM-as-a-judge signal that performs \emph{pairwise comparisons} between candidate patches, where the win-loss record is aggregated with a regularized Bradley-Terry model, and the resulting strength scores drive rank-based parent sampling inside a lightweight disaggregated tree search. 
Expensive downstream task evaluations are reserved only for the most promising nodes. 
Using a fully disaggregated tree search pipeline, the judge scores provide intermediate signal to guide exploration on promising candidate patches without being bottlenecked by slow evaluation runs.
SIFT outperforms existing tree-search based self-evolution frameworks on the full Polyglot benchmark with significantly lower resource requirements in terms of CPU hours, wall clock time, and API cost.
\end{abstract}

\section{Introduction}

Large Language Models (LLMs) have enabled the rapid development of autonomous agents, capable of reasoning, planning, and executing complex workflows. Among these, coding agents have made leapfrog gains within the past years, demonstrating the ability to solve increasingly difficult software engineering tasks~\citep{jimenez2024swebench, jain2025livecodebench}. One particular frontier within this domain is recursive self-improvement: the ability of an agent to modify its own source code to enhance its future performance.

Theoretical frameworks such as the G\"odel Machine \citep{godelmachine} propose that a program capable of provably beneficial self-modifications will converge to a global optimum. Recent empirical implementations, such as the G\"odel Agent \citep{godelagent}, SICA \citep{sica}, and the Darwin-G\"odel Machine (DGM) \citep{dgm}, have validated this potential, showing that agents can indeed navigate their own design space to refine their own implementation, with the improvement validated through competitive results on various benchmarks. However, this search process is extremely inefficient. Current methods rely on an expensive feedback loop where every candidate modification must be validated against a comprehensive benchmark suite to estimate its utility. As noted in recent works, this evolutionary process can incur astronomical costs, upwards of \$22,000 to evaluate on SWE-Bench, and consume thousands of CPU hours \citep{dgm, hgm}. 
This high barrier to entry makes research on self-improving agents less accessible and slows progress in the field.

In this paper, we identify the primary bottleneck in self-improvement as the \textbf{poor signal-to-cost trade-off} of existing evaluation methods. 
Full benchmark execution provides a reliable validation signal, but at high cost; evaluating on a small subset of benchmark tasks is cheaper, but yields a much noisier signal.
We propose that this process can be optimized by introducing intermediate, cheaper signals via a separate LLM judge tasked with rating the quality of a self-improvement proposal. We introduce Recursive \textbf{S}elf-\textbf{I}mprovement via \textbf{F}ast \textbf{T}ree Search (\textbf{SIFT}), a simple framework that dramatically reduces the time and financial cost of self-improvement while still enabling the discovery of meaningful self-improvement modifications, maintaining competitive performance gains.

Our key contributions include
\begin{enumerate}
    \item \textbf{Pairwise LLM-as-a-judge with Bradley--Terry aggregation.} We identify the per-patch evaluation as the time and cost bottleneck of the self-improvement step and replace it with pairwise LLM judgments. Each new candidate is compared against a set of incumbent agent harnesses and the win-loss record is aggregated with a regularized Bradley-Terry (BT) model that yields a global strength score for every node in the search tree, providing quick signal to expand the search tree without waiting on the expensive evaluation to complete.
    \item \textbf{Judge scores to guide speculative exploration in a fully disaggregated tree search.} We utilize a fully disaggregated pipeline to parallelize expansion and evaluation (Figure \ref{fig:disagg}). Parent nodes for the next round of self-improvement are sampled from the full archive with a weight combining BT rank, accuracy rank, and a visit-count exploration term. Simultaneously, we maintain a priority queue of next nodes to evaluate, ranked by the same aggregation metric. An orchestrator fetches the completed expansion nodes, inserts it into the evaluation queue to prioritize nodes with high potential while launching the next expansion step, dramatically speeding up tree search without being bottlenecked by slow evaluation processes, using judge signals as intermediate guides.
    \item \textbf{Empirical improvements and transferability of discovered agent harnesses across different coding models.} On the Polyglot benchmark, we achieve an improved score of 35\% using \texttt{o3-mini} and 32\% using \texttt{Qwen3-30B} as the base coding model, exceeding the performance of existing harnesses while maintaining a runtime under 250 CPU hours and a wall clock time of 7 hours for \texttt{Qwen3}, and under 50 CPU hours, 5 hours of wall clock time for \texttt{o3-mini} (Table \ref{tab:polyglot225}). Furthermore, we observe that the discovered agent harnesses transfer well across different coding models (Figure \ref{fig:polyglot_transfer}).
\end{enumerate}

\section{Related works}
Our work is closely related to an ongoing line of research on self-improving coding agents. The G\"odel machine~\citep{godelmachine} demonstrates that a program that is able to provably make useful modifications to its own code will find a globally optimum solution. The G\"odel agent~\citep{godelagent} proposes sensor-executor modules that implements the approximation of such a program in code, demonstrating improved performance across various language and reasoning benchmarks~\citep{dua2019dropreadingcomprehensionbenchmark, mmlu, gpqa, mgsm}. SICA~\citep{sica} enforces the self-improvement agent to identically match the agent used during benchmark evaluations, and shows significant performance gains in coding~\citep{jimenez2024swebench, jain2025livecodebench}. STOP~\citep{zelikman2024selftaught} uses a meta-utility function to guide a seed improver to discover multiple self-improvement strategies, such as simulated annealing and genetic algorithms. The Darwin-G\"odel Machine~\citep{dgm} introduces open-endedness by maintaining an archive of well-performing agents to improve diversity in the agentic design search space. The Huxley-G\"odel machine~\citep{hgm} proposes the Clade Meta Productivity~\citep{cmp_huxley} metric to sample the next agent based on the performance of a parent and all its descendants. More recently, Group Evolving-Agents~\citep{weng2026groupevolvingagentsopenendedselfimprovement} implement crossover between different branches in the search tree via applying the self-improvement step within a group of leaf nodes. Red Queen G\"odel Machine~\citep{iacob2026redqueengodelmachine} makes use of a reviewer agent on each Polyglot task to supply auxiliary signal before running the full Polyglot evaluation and evolves the reviewer agent together with the coding agent. In SIFT, we apply the the judge signal for comparison between different nodes in the tree search, obtaining a faster signal before any task evaluations even complete, allowing for efficient exploration during the tree search process.

More broadly, this line of work on self-improving coding agents aims to efficiently explore the design space of harness design for agents. ADAS~\citep{adas} demonstrates a Meta Agent Search algorithm that exceeds SOTA performance of human-designed agentic systems across several comprehension and reasoning tasks. AFlow~\citep{zhang2025aflow} explores the design space of code-based workflows using Monte Carlo Tree Search, simultaneously achieving performance improvements and cost reductions against stronger LLM backbones. Such work has also been extended to multi-agent systems via MaAS~\citep{maas}.

Pairwise preferences and the Bradley--Terry model~\citep{bradleyterry1952} have a long history in ranking under noisy feedback, and are used pervasively in LLM-as-a-judge pipelines~\citep{llmasajudge, codejudebench} because pairwise calls are calibration-free and more reliable than absolute rubric scores. We import this machinery into self-improvement: instead of asking the judge for an absolute score and taking top-$k$, we accumulate pairwise wins across the tree and fit a regularized BT model whose ranks drive parent sampling.

\section{Methods}
\begin{figure}[hbt]
\vspace{-6mm}
\begin{centering}
    \includegraphics[width=\linewidth]{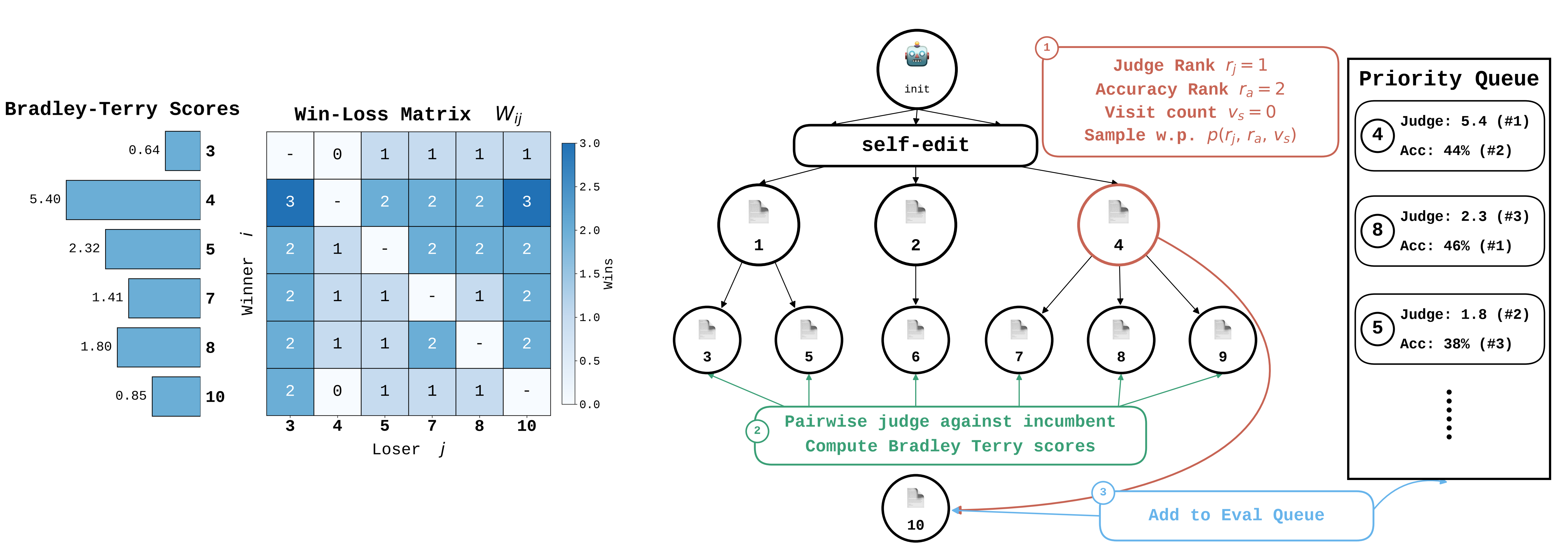}
    \caption{Illustration of the SIFT pipeline. In a single expansion step, node $i$ is sampled with probability $P(i) \propto \exp(-(\alpha r_{b}(i) + \beta r_{a}(i) + \eta \log(1+v_i)))$, self-improves, then judged against previous versions in the existing agent archive. The win/loss record is updated and the BT scores of each node is solved numerically. Then the new agent is inserted into the priority queue for evaluation on downstream tasks, prioritizing higher potential nodes.}
    \label{fig:pipeline}
\end{centering}
\vspace{-6mm}
\end{figure}

\paragraph{Recursive Self-Improvement Loop.}
A self-improving coding agent is a recursive loop between two roles. The coding model is the agent's runtime backbone, reading source, editing files, running tests, and invoking tools to solve downstream tasks. The self-improving model sits one level above: given the current code and diagnostic context from prior runs (e.g. evaluation logs and failure traces), it identifies weaknesses and writes a patch, producing a child agent. The child is then evaluated by running the coding model under the new agent harness, and the resulting signal from downstream tasks feeds the next round of self-modification. Because each patch reshapes the very code that orchestrates the coding model, improvements compound recursively.
Prior systems including SICA \citep{sica}, DGM \citep{dgm}, and HGM \citep{hgm} instantiate this loop, differing primarily in how they search the resulting tree of self-modifications.
In practice, the two roles can be filled by different LLMs, often with a cheaper backbone handling the many coding-model calls during evaluation.
This loop is powerful but expensive: each candidate patch must be benchmarked before the next round, and benchmark runs dominate wall-clock time.

To address this, we introduce \textbf{SIFT} (Recursive \textbf{S}elf-\textbf{I}mprovement via \textbf{F}ast \textbf{T}ree-search), a sample-efficient framework that decouples exploration from expensive benchmark evaluation. 
SIFT maintains an archive of candidate agent harnesses and uses cheap LLM-based judge feedback to decide which branches of the self-improvement tree to expand. Each newly generated node is compared pairwise against a small set of strong existing nodes. The resulting win-loss records are aggregated with a regularized Bradley-Terry (BT) model, producing a judge-based strength score for every node. During search, parent nodes are sampled according to a rank-based mixture of judge-based strength, subset benchmark accuracy, and a visit-count penalty that encourages exploration using the following: 
\begin{equation}
\label{eq:prob}
P(i) \propto
\exp\!\left(
-\alpha r_b(i)
-\beta r_a(i)
-\eta \log(1+v_i)
\right),
\end{equation}
where $\alpha$ and $\beta$ control the influence of judge and accuracy ranks, and $\eta$ controls the strength of the visit-count exploration penalty.
Thus, expensive benchmark evaluations are reserved 
for candidates that achieve high preliminary accuracy and are judged promising by the LLM.

\paragraph{LLM-as-a-judge.}
After a new child node is generated, SIFT estimates its potential utility using a judge model before running a benchmark evaluation on it. The judge never sees benchmark tasks or task outcomes. Instead, it compares the new candidate against one existing agent at a time, using candidate agent implementations as input.
In Table~\ref{tab:judge_prompt_smoke_override}, we compare several input formats for the judge model. 
The ``Diffs'' variant provides the root-agent implementation followed by the chain of \texttt{diff}s leading to the final implementation of each agent, while the ``Full files'' variant provides the complete source files for the two candidate agents. 
Each comparison returns a preference between the two agents, which SIFT records as a single win or loss.

\paragraph{Comparison-target selection.}
To keep judging cost bounded as the tree grows, each new candidate is compared against only a small set of existing nodes. These comparison targets are chosen from the current archive by favoring agents that already look strong under the same signals used for parent selection: judge rank and benchmark-accuracy rank. In practice, we limit our comparison to the top 10 nodes in the archive. This concentrates judge effort on comparisons against relevant frontier nodes rather than repeatedly comparing against weak or already-dominated agents.

\paragraph{Bradley--Terry aggregation.}
The Bradley--Terry (BT) model is a classical model for converting noisy pairwise comparisons into a global ranking~\citep{bradleyterry1952}. In our setting, each node in the search tree corresponds to an agent version, and each judge call produces a pairwise preference between two such versions. BT assumes that every node $i$ has a latent positive strength parameter $\theta_i > 0$, and that the probability that node $i$ is preferred to node $j$ is
\[
    P(i \succ j)
    =
    \frac{\theta_i}{\theta_i + \theta_j}.
\]
This gives a principled way to pool sparse and potentially inconsistent judge preferences: rather than treating each comparison independently, all wins and losses are explained jointly by a single set of latent strengths.

SIFT aggregates judge preferences in a global win-count matrix $W$, where $W_{ij}$ is the number of times node $i$ has been preferred to node $j$. After each judging round, we refit a regularized BT model to convert the accumulated pairwise outcomes into one judge-strength score for every node. 
To account for limited comparison data for newly generated nodes we include a regularizer $\lambda$ which acts as a pseudo-count.
We further normalize all scores so that $\sum_i \theta_i = n$, where $n$ is the number of candidate patches.
The resulting BT strengths are used only through their induced ranks in the parent-sampling rule below, which makes the search less sensitive to the absolute scale of the fitted scores.

\paragraph{Rank-based parent sampling.}
\label{sec:sample}
SIFT uses judge feedback only through ranks, not raw judge scores. 
Using the measured Bradley-Terry score and benchmark accuracy we derive measures of the ranking denoted $r_b(i)$ and $r_a(i)$ respectively.
If a node is still waiting for evaluation, it temporarily inherits its parent's accuracy for sampling purposes. Parent nodes are then sampled from the archive according to \eqref{eq:prob} where $v_i$ is the number of times node $i$ has already been selected as a parent. The first two terms bias search toward nodes that look strong under judge and benchmark signals. The visit-count term downweights nodes that have already been explored many times, encouraging the search to continue branching rather than repeatedly refining the same lineage.

\paragraph{Disaggregated pipeline.}
\begin{figure}
\centering
\includegraphics[width=0.75\linewidth]{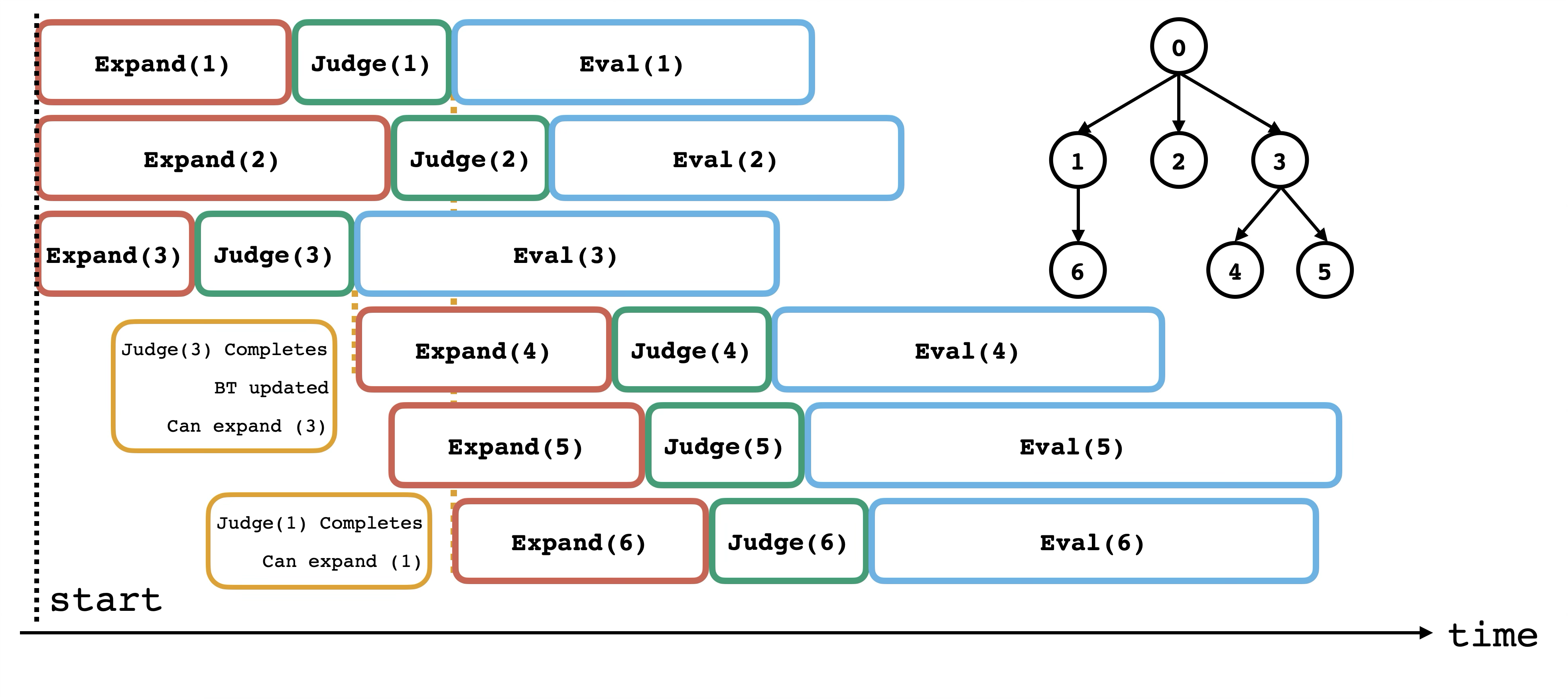}
\caption{Illustration of the disaggregated pipeline. SIFT allows expansions of nodes that have yet to complete full evaluation but received a strong signal from the judge.}
\label{fig:disagg}
\end{figure}
SIFT runs a non-blocking master loop (Algorithm~\ref{alg:sift}, Figure~\ref{fig:disagg}). On each iteration, it (1) samples a new node to expand according to Eq.~\ref{eq:prob} and runs it through an easy subset to filter out broken patches immediately, (2) compares the improved agent against strong incumbent agents and refits the BT model to obtain the BT scores $b_i$ and ranks $r_b(i)$ sorted by BT strength, and (3) appends the new agent into the evaluation priority queue, sorted by $r_b(i) + r_a(i)$, where nodes that pass the easy subset inherit $a_i$ from their parents. The evaluator and expansion runs proceed in parallel, allowing expansions to continue without waiting for full evaluation to complete.

Similar implementations of such disaggregated pipelines have also been explored in \citep{hgm}, which decides which nodes to expand and which to evaluate through Thompson sampling. However, we observe that this implementation often leads to inaccurate estimations of the final accuracy due to random sampling of tasks used for evaluation, where nodes that received easier tasks would have high scores simply just by chance. Furthermore, the signal obtained through a single downstream task evaluation is extremely noisy, and still requires significant amount of time to complete. To get the best of both worlds of (1) high signal to discriminate between candidate nodes and (2) fast tree search that is not bottlenecked by downstream evaluations, SIFT makes use of the judge signal to guide the expansion process before evaluation runs complete. We use a fixed subset to compare accuracies between nodes in order to provide fair evaluations for each node that is comparable between nodes, yet maintaining efficiency by allowing fully disaggregated expansion and evaluation to run in parallel.

\subsection{Cost estimation of self-evolution.}
Each node in the search tree records three cost phases: \texttt{expansion}, \texttt{judge} and \texttt{eval}. At default settings with maximum pairwise comparisons per node of $K=10$, a single node costs roughly 10 pairwise calls at a few cents each, which is an order of magnitude cheaper than a full Polyglot-50 evaluation (see Table~\ref{tab:cost_analysis}). Existing self-improvement frameworks such as DGM make use of the benchmark evaluation across the 50 problem subset, which accounts for majority of the cost in the evolution process\footnote{Authors of previous works typically switch to a more cost-efficient coding model to lower the overall cost of the tree search, dominated by the downstream task evaluation}. The use of a judge model allows a much faster and cost-efficient way to obtain an intermediate signal of the strength of the child node.

\begin{table}[ht]
\centering
\caption{Resource consumption breakdown per step. Self-improve and LLM-judge costs and durations are reported per task, averaged across 50 tasks, using \texttt{gpt-5.4}. A single judge \emph{comparison} is one call on a pair; a full judging round for one new node costs up to $K=10$ such calls.}
\begin{tabular}{lrrr}
\toprule
\textbf{Module} &  \textbf{LLM} & \textbf{Cost (USD)} & \textbf{Time (CPU Hours)} \\
\midrule
Self-Improve (expansion) & \texttt{gpt-5-mini}      & 0.12   & 0.186  \\
LLM-judge (pairwise call)  & \texttt{gpt-5.4}   & 0.044  & 0.0042 \\
Polyglot-50 full eval   & \texttt{o3-mini}        & 6.0    & 2.6    \\
\bottomrule
\end{tabular}
\label{tab:cost_analysis}
\end{table}

\vspace{-4mm}
\section{Experiments}
We report our key results across various coding benchmarks: SWE-bench\citep{jimenez2024swebench}, Polyglot~\citep{polyglot} and Terminal Bench~\citep{merrill2026terminalbenchbenchmarkingagentshard}. SIFT is built on top of the DGM~\citep{dgm} harness, and we report comparisons to the DGM and HGM~\citep{hgm} baselines.
\vspace{-4mm}

\begin{figure}[hbt]
\centering

\begin{subfigure}[t]{0.48\linewidth}
    \centering
    \includegraphics[width=\linewidth]{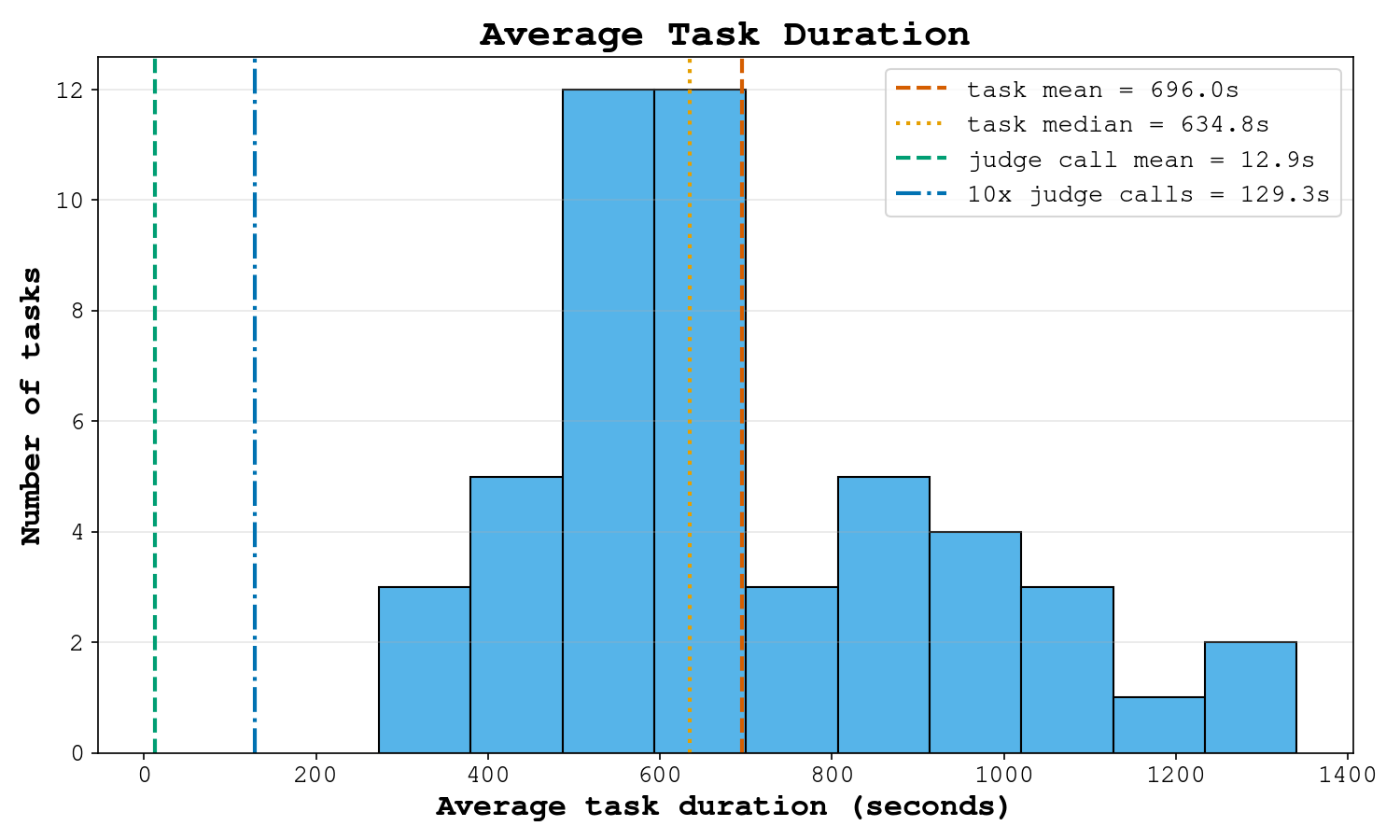}
    \caption{Task duration distribution on Polyglot. The average set of pairwise judge comparisons against the top 10 incumbent patches complete in a significantly shorter period of time than even the first evaluation.}
    \label{fig:task_duration_histogram}
\end{subfigure}
\hfill
\begin{subfigure}[t]{0.48\linewidth}
    \centering
    \includegraphics[width=\linewidth]{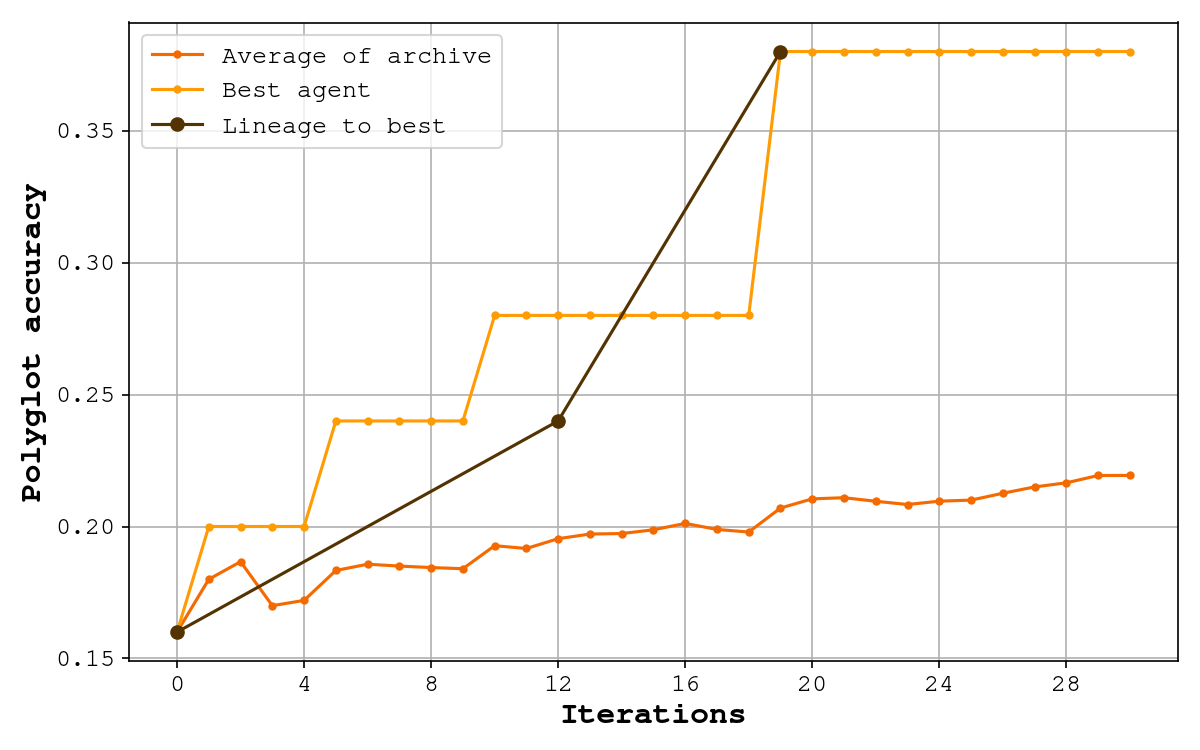}
    \caption{Evolution progress of SIFT using \texttt{Qwen3-30B} as the base coding model and \texttt{Qwen3-480B} as the self-improvement and judge model.}
    \label{fig:qwen_qwen_plot_progress}
\end{subfigure}

\caption{Evolution progress of SIFT. We plot Polyglot-50 accuracy against iteration number, running for 30 evolution steps, improving from an accuracy of 16\% to 38\% on Polyglot-50. The overall average accuracy of the archive across Polyglot-50 steadily increases with iteration number. The best descendant achieves a final Polyglot accuracy of 31.1\%.}
\label{fig:qwen_qwen_plot}

\end{figure}

\subsection{Experimental setting.}
\paragraph{Polyglot}

Aider Polyglot~\citep{polyglot} is a multi-language benchmark across \texttt{c++, go, rust, java, javascript, python}. The full benchmark has 225 tasks. 
Our models and harness in our experiments follow DGM~\citep{dgm} and HGM~\citep{hgm}.  
In the search process, DGM and HGM use different intermediate signals as indicators of the strength of a node: DGM evaluates nodes on sets of increasing size, beginning with a small subset of 10 tasks to filter out catastrophically bad patches, followed by a medium subset of 50 tasks, and reserves the full 225-task benchmark for held-out final evaluation.
In contrast, HGM dynamically adjusts the number of tasks evaluated for each node, sampling tasks at random from the full benchmark.
We adopt the DGM evaluation protocol, but use a smaller 4-task gate, followed by a 50-task subset, which we refer to as Polyglot-50, as the intermediate search signal. 
We reserve the full 225-task benchmark for held-out final evaluation.
All experiments run in sandboxed Docker containers.

We evaluate SIFT against existing self-improvement harnesses under matched settings on Polyglot-50. 
Following the setup in \citet{hgm}, we use Qwen3-Coder-30B-A3B-Instruct \citep{qwen3technicalreport} (referred to as \texttt{Qwen3-30B} throughout) as the coding model, and \texttt{Qwen3-480B} as both the expansion model and the LLM judge. 
For the DGM setup, the coding model is \texttt{o3-mini} and the self-improvement model is \texttt{gpt-5-mini}. The original DGM paper uses \texttt{Claude-3.5-Sonnet-v2} for the self-improvement role, but it is now deprecated; we substitute \texttt{gpt-5-mini} as a comparable proxy (Polyglot accuracy 34.7\% vs. 32.0\% for Claude-3.5-Sonnet-v2). 
Our independent evaluations of the initial harness matches those reported by the original authors, with DGM achieving 14.2\% and HGM achieving 20.0\%.
To compare fairly under the constraint of 800 evaluations in ~\cite{hgm}, we limit the total number of expansion steps to 30 steps. The evolution progress is shown in Fig ~\ref{fig:qwen_qwen_plot}, where the best found node achieves an accuracy of 31.1\% after a total of 34.3USD in API cost and a wall clock time of 6.71h with a total of 224 CPU hours, a tenth of that of the baseline of DGM. This also achieves a higher accuracy at a lower CPU hour count than a different asynchronous tree search implementation of HGM which reaches 30.5\% accuracy at 347 CPU hours. Our accuracies also surpass that of SICA. The full results are shown in Table ~\ref{tab:polyglot225}.

\begin{table}[ht]
\centering
\caption{Polyglot-225 accuracy and comparison to several baselines across different coding and judge models. SIFT outperforms other evolution frameworks such as SICA, DGM and HGM.}
\label{tab:polyglot225}
\small
\setlength{\tabcolsep}{6pt}
\begin{tabular}{llcc}
\toprule
\textbf{Method} & \textbf{Coding model} & \textbf{Judge model} & \textbf{Polyglot Accuracy (\%)}  \\
\midrule
Base Agent & \texttt{Qwen3-Coder-30B} & \texttt{None} & 20.0 \\
SICA       & \texttt{Qwen3-Coder-30B} & \texttt{None} & 25.1 \\
DGM        & \texttt{Qwen3-Coder-30B} & \texttt{None} & 27.1  \\
HGM        & \texttt{Qwen3-Coder-30B} & \texttt{None} & 30.5  \\
 SIFT      & \texttt{Qwen3-Coder-30B} & \texttt{Qwen3-Coder-480B} & 31.1 \\
\textbf{SIFT}       & \texttt{Qwen3-Coder-30B} & \texttt{gpt-5.4} & \textbf{32.0}  \\

\specialrule{0.08em}{3pt}{3pt}

Base Agent & \texttt{o3-mini} & \texttt{None} & 14.2 \\
DGM        & \texttt{o3-mini} & \texttt{None} & 30.7 \\
SIFT       & \texttt{o3-mini} & \texttt{None} & 29.8 \\
\textbf{SIFT}       & \texttt{o3-mini} & \texttt{gpt-5.4} & \textbf{35.1}  \\
\textbf{SIFT}       & \texttt{o3-mini} & \texttt{gpt-5-mini} & 31.6  \\
\bottomrule
\end{tabular}
\end{table}

\begin{table}[ht]
\centering
\caption{Run-level resource use for the three SIFT configurations reported in Table~\ref{tab:polyglot225}. Judge cost is included in total API cost. Each row is a single search run.}
\label{tab:polyglot_run_costs}
\small
\resizebox{\linewidth}{!}{%
\begin{tabular}{lrrrrr}
\toprule
\textbf{Configuration} & \textbf{Expansion steps} & \textbf{API cost} & \textbf{Of which judge} & \textbf{CPU-h} & \textbf{Wall-clock} \\
\midrule
\texttt{Qwen3-30B} / \texttt{Qwen3-480B} judge & 30 & \$34.3 & \$4.1 & 224 & 6.7 h \\
\texttt{Qwen3-30B} / \texttt{gpt-5.4} judge & 30 & \$33.7 & \$9.7 & 188 & 6.8 h \\
\texttt{o3-mini} / \texttt{gpt-5.4} judge & 20 & \$86.8 & \$15.7 & 59 & 2.1 h \\
\bottomrule
\end{tabular}%
}
\end{table}

\begin{figure}[hbt]
\centering
\begin{subfigure}[t]{0.48\linewidth}
    \centering
    \includegraphics[width=\linewidth]{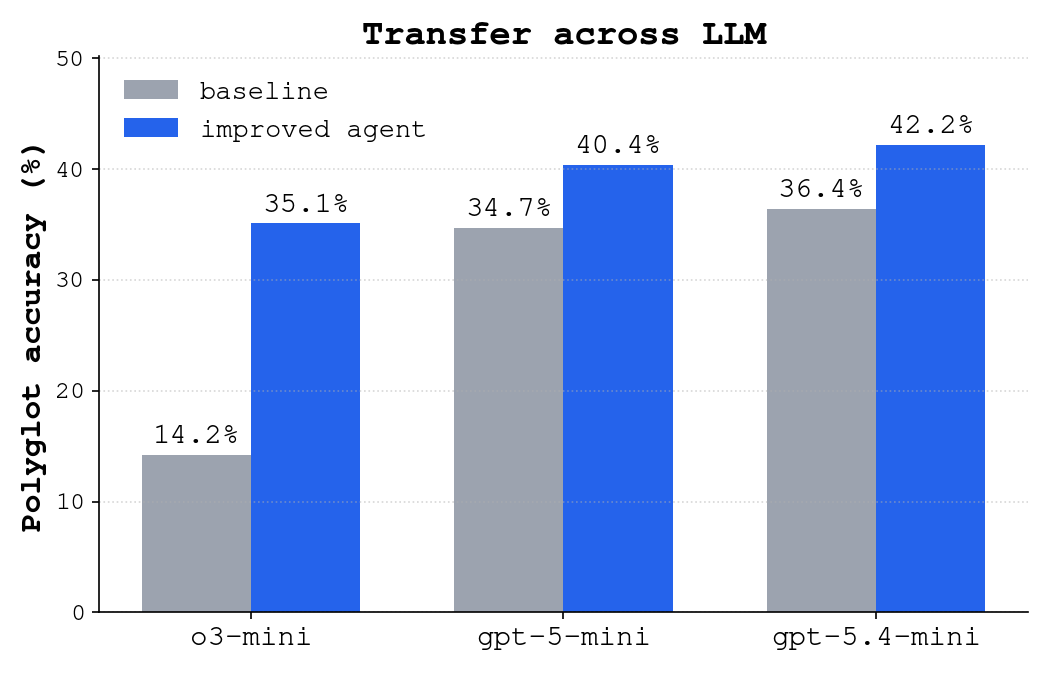}
    \caption{\texttt{o3-mini} to other LLMs}
    \label{fig:o3_transfer}
\end{subfigure}
\hfill
\begin{subfigure}[t]{0.48\linewidth}
    \centering
    \includegraphics[width=\linewidth]{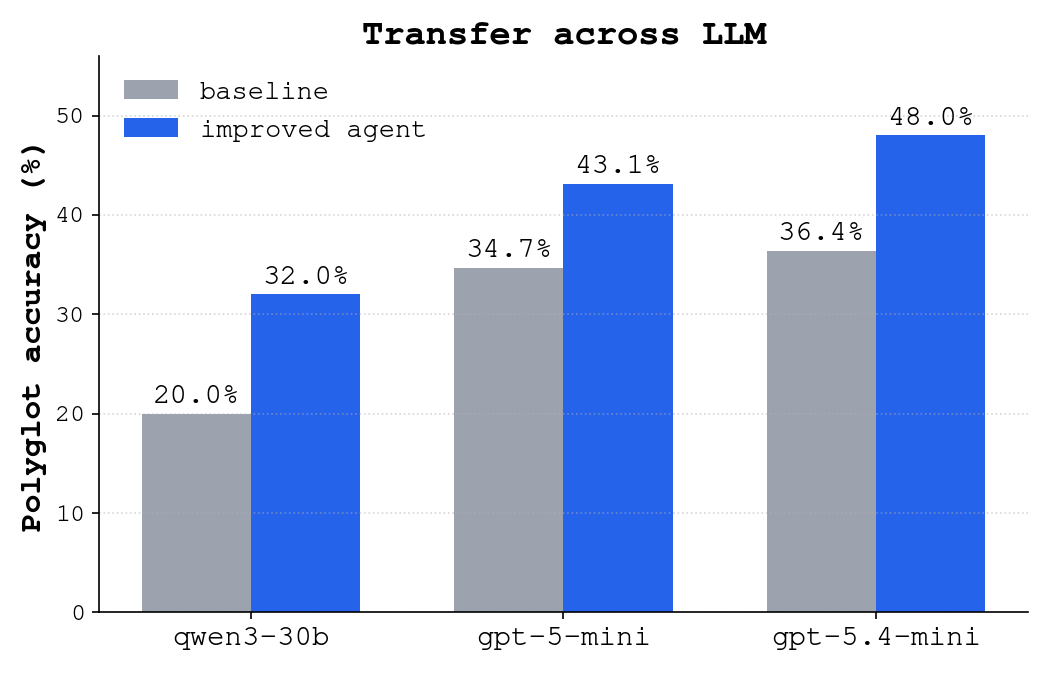}
    \caption{\texttt{Qwen3} transfer to other LLMs}
    \label{fig:qwen3_transfer}
\end{subfigure}
\caption{Transfer between LLMs of the best Polyglot agent found by SIFT. The improved agents in SIFT generalizes to \texttt{gpt-5-mini} and \texttt{gpt-5.4-mini}, where the improved agent consistently achieves a higher accuracy than the base coding agent.}
\label{fig:polyglot_transfer}
\vspace{-4mm}
\end{figure}

Furthermore, the proposed judge plus disaggregation workflow also transfers to other starting harnesses. Using the DGM initial harness with \texttt{o3-mini} as the base coding agent, we achieve an improvement to 35.1\% within 30 steps of expansion, exceeding the 30.7\% result obtained by DGM after 80 nodes of tree search. This experiments completes in under 5 hours of wall clock time, 42 CPU hours and 150 USD of API credits. We repeat the experiment on \texttt{o3-mini} 3 more times and obtain the results in the range $[32.0\%, 35.6\%]$, consistently outperforming the DGM evolution framework with lower resource consumption.

We also observe that these improvements transfer across LLMs. We take the improved agent harnesses discovered by SIFT with \texttt{o3-mini} and \texttt{Qwen3-30B} as the base coding model. 
We re-evaluate these agent harnesses with \texttt{gpt-5-mini} and \texttt{gpt-5.4-mini} as the coding model on Polyglot. 
As shown in Figure~\ref{fig:polyglot_transfer}, the SIFT-discovered harnesses consistently improve Polyglot performance. 
This shows that SIFT is not only fast and efficient, but also discovers robust and transferable agent harnesses, rather than merely optimizing for the particular coding model used to obtain benchmark feedback during tree search.

\subsection{TerminalBench 2.1}
\label{app:terminalbench}
We next evaluate SIFT on TerminalBench~\citep{merrill2026terminalbenchbenchmarkingagentshard}, a heterogeneous benchmark of long-horizon tasks executed and graded inside task-specific terminal containers. 
Again we use \texttt{gpt-5-mini} as the coding model, \texttt{gpt-5} for diagnosis and self-improvement, and \texttt{gpt-5.4-high} as the pairwise judge. 
During search, candidate evaluations use a fixed randomly sampled 50-task subset and we employ more stringent timeout requirements.
Namely the maximum time we allow for each agent is the minimum of the TerminalBench task timeout or 30 minutes. TerminalBench tasks vary in their allowed 

\begin{table}[t]
\centering
\small
\caption{TerminalBench selection comparison. ``Search eval'' is the score observed on the 50-task search subset. Repeated means average three independent full evaluations of each selected agent.}
\label{tab:terminalbench_picks}
\begin{tabularx}{\linewidth}{ccc>{\centering\arraybackslash}X}
\toprule
Setting & Agent & Search eval & Repeated mean \\
\midrule
Initial agent & Starting point & 14/50 & 26.0/89 (29.2\%)\\
SIFT with judge & Judge rank 1 
& 18/50 & 32.7/89 (36.7\%)\\
SIFT with judge & Accuracy rank 1 
& 19/50 & 25.0/89 (28.1\%) \\
No-judge ablation & Accuracy rank 1 
& 19/50 & 26.0/89 (29.2\%) \\
\bottomrule
\end{tabularx}
\end{table}

We compare one judge-guided search and one no-judge ablation, both capped at 30 expansions and initialized from the same starting agent which scores 14/50.
At the end of the search we selected three agents and performed three full evaluations on them: the highest scoring agent of the SIFT run, the agent the judge most preferred from the SIFT run, and the highest scoring agent from the no-judge ablation.
We present these results in Table~\ref{tab:terminalbench_picks}. Here we see that although the judge-selected nodes performed slightly worse during search, it exceeded both other agents on the complete benchmark under repetition.

\begin{figure}[!htb]
\centering
\includegraphics[width=\linewidth]{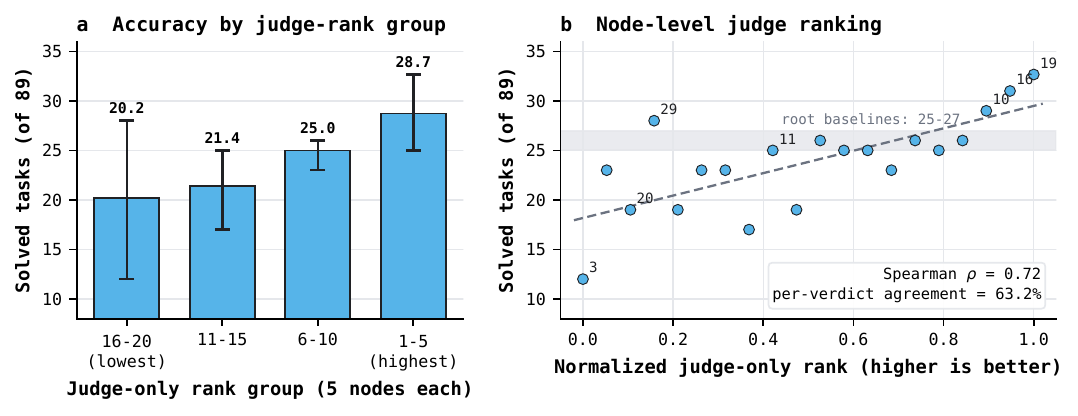}
\caption{TerminalBench judge diagnostics for the fully evaluated agents in the clean judge-guided run. \textbf{Left:} nodes are sorted by judge-only Bradley-Terry rank and divided into four groups; bars show each group's mean realized accuracy and points show its individual nodes. \textbf{Right:} the underlying node-level judge rank versus realized accuracy.}
\label{fig:terminalbench_appendix}
\end{figure}

To further measure the ranking ability of the judge, we fully evaluated each agent in the judge-guided search on the full benchmark and computed the Spearman correlation.
In Figure~\ref{fig:terminalbench_appendix} we present these results.

\paragraph{Judge model strength.}
We repeat the judge-guided search with the weaker \texttt{gpt-5} as the judge, keeping the coding, diagnosis, and self-improvement models fixed, and evaluate its archive on the full benchmark in the same way.
Table~\ref{tab:terminalbench_judge_strength} compares the two judge models and the no-judge ablation.
The weaker judge retains most of the global ranking signal: its BT rank correlates with full-benchmark score at $\rho=0.71$ across its 20 fully evaluated candidates, essentially matching \texttt{gpt-5.4-high}, and its search still produces an agent that repeats at 34.5\%, well above the no-judge search.
However, we find the performance of the two judges separate at the top of the ranking: within the BT top-5, \texttt{gpt-5}'s pairwise agreement with the ground-truth ordering falls to 0.50 against 1.00 for \texttt{gpt-5.4-high}, and unlike \texttt{gpt-5.4-high} its top-ranked candidate is not the best agent its search produced.
A weaker judge therefore still steers the search toward strong candidates but is less reliable for the final selection among them. This suggests an alternative tiered configuration in which a cheaper judge performs the bulk of comparisons and a stronger judge orders the frontier may be viable.

Taken together these results suggest that the judge is most useful when it comes to providing a ranking signal rather than as a surrogate for benchmark accuracy. Search-time accuracy alone does not reliably identify the best agent, whereas BT ranking appears to recover a substantially stronger candidate and remains informative across judge models. The stronger judge is particularly valuable near the top of the frontier, where small ranking errors directly affect which final agent is selected.

\newpage

\begin{table}[t]
\centering
\small
\caption{TerminalBench judge-model comparison, one search run per condition. ``Best agent found'' is the highest-scoring agent in each run's archive, averaged over three independent full-benchmark evaluations, and ``Gain'' is its percentage-point improvement over the starting agent (29.2\%). Recall@5 counts members of the ground-truth top-5 present in the judge's BT top-5; $\rho$ is the Spearman correlation between BT rank and full-benchmark score over the run's 20 fully evaluated candidates; pairwise (top-5) is agreement with the ground-truth ordering within the BT top-5.}
\label{tab:terminalbench_judge_strength}
\begin{tabular}{lccccc}
\toprule
Judge & Best agent found & Gain & Recall@5 & $\rho$ & Pairwise (top-5) \\
\midrule
\texttt{gpt-5.4-high} & 36.7\% & +7.5 & 4/5 & +0.72 & 1.00 \\
\texttt{gpt-5} & 34.5\% & +5.3 & 3/5 & +0.71 & 0.50 \\
None & 29.2\% & +0.0 & --- & --- & --- \\
\bottomrule
\end{tabular}
\end{table}

\subsection{Analysis}
\label{sec:ablations}

\paragraph{SIFT Speedups}
\begin{wrapfigure}{r}{0.58\textwidth} 
\centering 
\includegraphics[width=\linewidth]{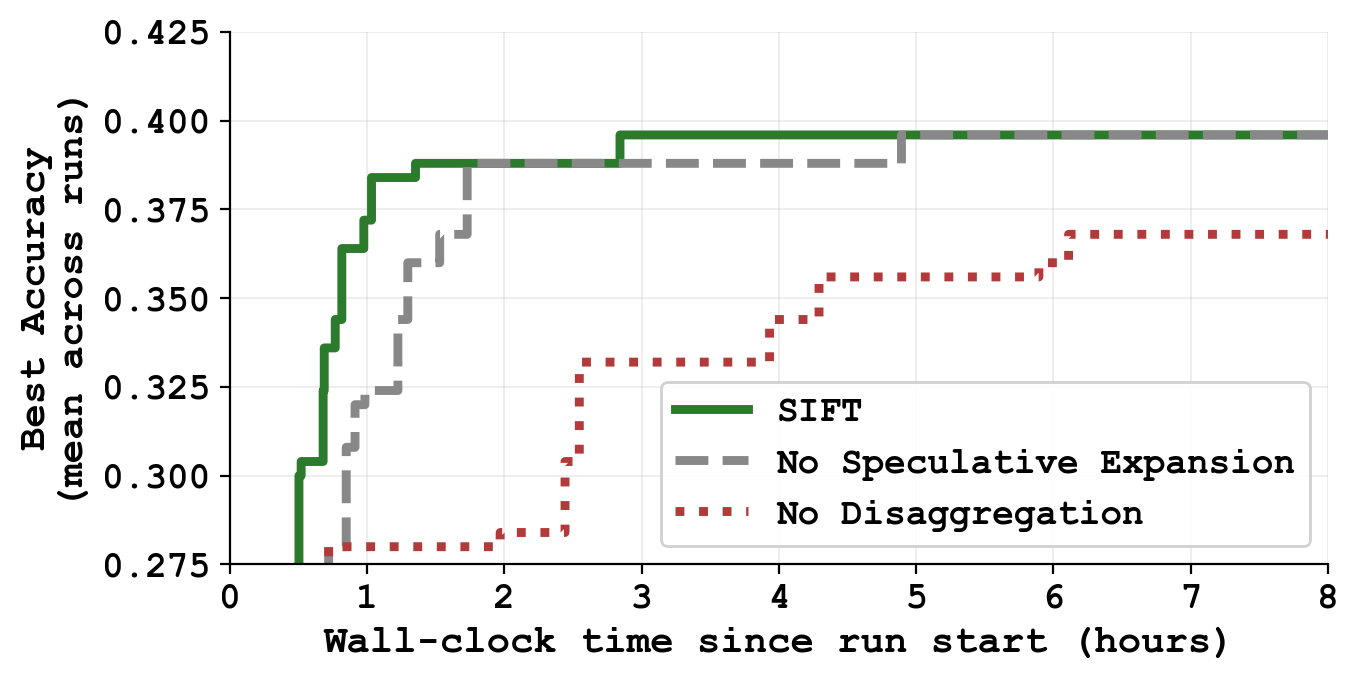} 
\caption{Speedups in SIFT, averaged across five runs.} 
\label{fig:speedups} 
\vspace{-2mm} 
\end{wrapfigure}
SIFT utilizes two main ingredients to achieve speedups across separate runs: the disaggregation pipeline to allow for parallel expansion and evaluation, and allowing the speculative expansion of nodes that are judged but not fully expanded. 

Figure ~\ref{fig:speedups} show the efficiency gains from each of these components, averaged across five runs of SIFT. We observe that the asynchronous pipeline accounts for majority of the speedups, and on top of this, the speculative expansion of unevaluated parent nodes according to the judge rankings allow SIFT to arrive at a higher accuracy in a shorter period of time, further improving tree search efficiency within the disaggregated pipeline.

\paragraph{Qualitative insights}
\begin{wrapfigure}{r}{0.58\textwidth} 
\centering 
\includegraphics[width=\linewidth]{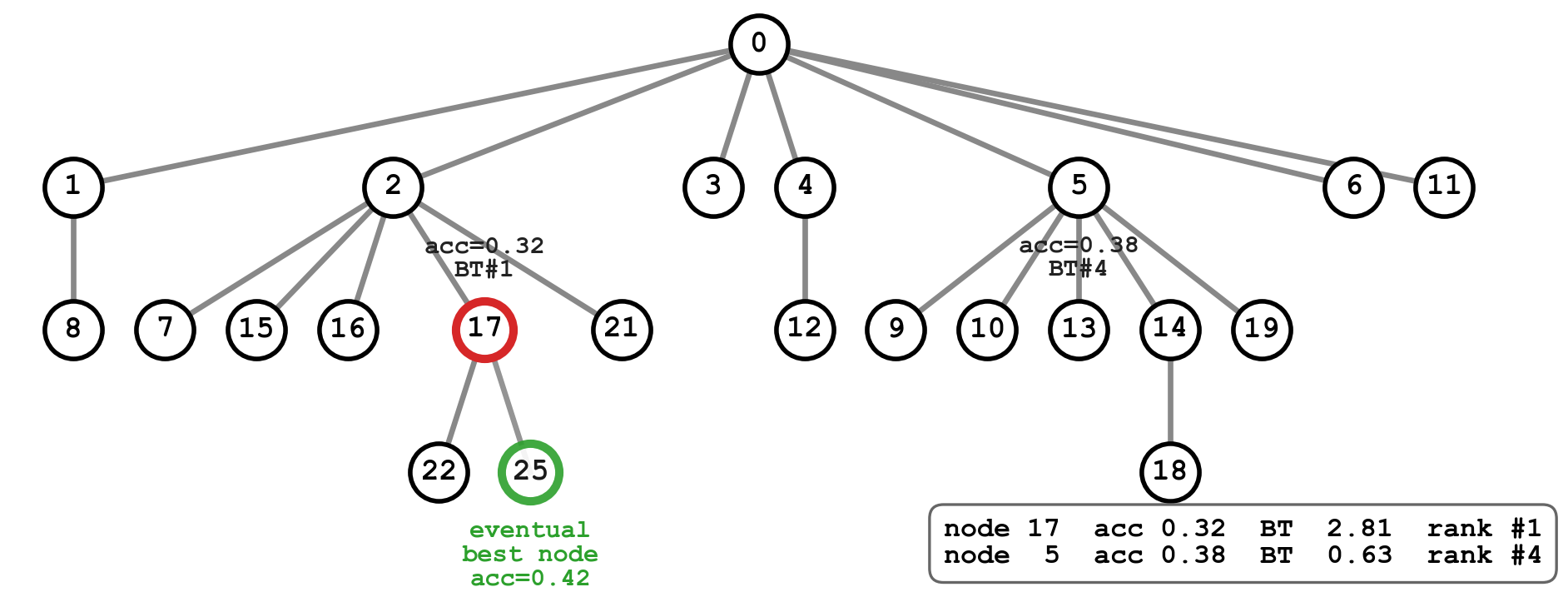} 
\caption{Qualitative insights into judge preferences.} \label{fig:qualitative} 
\vspace{-6mm} 
\end{wrapfigure}

In an open-ended tree search, the judge provides guiding signal to explore lower accuracy but more promising nodes instead of exploiting the best found accuracy. 

We demonstrate in Figure~\ref{fig:qualitative} a specific instance in which the judge model provides a intermediate guiding signal to explore a promising node. At the specific snapshot of the expansion process of the tree illustrated in Figure~\ref{fig:qualitative}, node 17, which is ranked lower in terms of accuracy, over other existing nodes, achives the highest BT judge score of 2.81, prioritizing it for further expansion. The expansion of node 17 then spawns a children node 25, which eventually achieves the highest accuracy of 42\%. In particular, at the instant of completion of the tree search expansion to 30 nodes, the full lineage of the nodes \texttt{2} $\to$ \texttt{17} $\to$ \texttt{25} achieves BT judge scores of 1.484, 1.93, and 1.93 (ranking 2, 1, and 1) respectively, showing that the judge consistently ranks the lineage highly, enabling recursive self-improvement along a promising exploration direction. 

For TerminalBench 2.1 we found the judge preferred a gap between the highest performing node under full-evaluation and the highest performing node during search. For the latter we observed the main addition was an instruction-derived local verifier with multi-attempt refinement. However this was gated behind an environment flag that defaults to off, and its rewritten bash tool started a fresh interactive shell on every call while claiming persistent state. The judge identified both properties from the code alone, noting that the verifier ``likely does nothing on default runs'' and that the shell rewrite is ``a major runtime risk on the actual execution path.'' The highest performing node developed a persistent shell with explicit timeouts, process-group cleanup and output truncation, plus bounded prompt changes that encourage at most two validation passes. In this case the failure mode the judge caught was invisible to any downstream evaluation score and thus made selection by highest accuracy unreliable here.

\paragraph{Judge input format.}
We ablate how candidate agents are presented to the LLM judge. Since SIFT uses judge scores only through their induced ranks, the most important property of the judge is not calibration, but whether its ranking agrees with downstream benchmark performance. We therefore measure the Pearson and Spearman correlation between each node's Bradley--Terry score and its realized Polyglot accuracy across 50 non-root nodes from the \texttt{Qwen3-Coder-30B} Polyglot tree-search run.

Table~\ref{tab:judge_prompt_smoke_override} shows that the representation of the candidate patch has a substantial effect on judge quality. Providing only the sequence of diffs gives a useful but relatively weak ranking signal, with Spearman correlation around $0.40$. Adding a swap-order pass, where the same pair is judged in both presentation orders to ensure symmetry across the presentation orders to break any positional biases, slightly improves rank correlation but roughly doubles judging cost. In contrast, showing the judge the resulting full files yields a much stronger signal: Spearman correlation increases to $0.68$, while the cost remains far below a benchmark evaluation. The full-file representation likely helps because it lets the judge reason about the final behavior of the agent directly, rather than reconstructing that behavior from a chain of patches. Swap-ordering full files gives nearly the same rank correlation, suggesting that most of the gain comes from the richer program context rather than from reducing presentation-order bias. We adopt the full-file format without swap for our implementation.

\begin{table}[ht]
\vspace{-5mm}
\centering
\caption{
Judge-input ablation for \texttt{Qwen3-Coder-480B} on the Polyglot tree-search run in Fig.~\ref{fig:qwen_qwen_plot}. 
We report the correlation between Bradley--Terry judge score and realized Polyglot accuracy over 50 non-root nodes. 
Spearman correlation is the primary metric because SIFT uses rank-based parent sampling.
}
\label{tab:judge_prompt_smoke_override}
\small
\setlength{\tabcolsep}{5pt}
\begin{tabular}{lccc}
\toprule
\textbf{Judge input variant} & \textbf{Pearson $r$} & \textbf{Spearman $\rho$} & \textbf{Cost per comparison (USD)} \\
\midrule
Diffs                    & $+0.14$ & $+0.40$ & $0.0076$ \\
Diffs + swap-order       & $+0.11$ & $+0.43$ & $0.014$  \\
\textbf{Full files}      & $\mathbf{+0.45}$ & $\mathbf{+0.68}$ & $0.011$ \\
Full files + swap-order  & $+0.37$ & $+0.67$ & $0.021$ \\
\bottomrule
\end{tabular}
\end{table}

These results motivate our default judge configuration: we use full-file pairwise comparisons without swap-ordering. This choice gives the strongest rank signal per dollar among the tested variants and preserves the main advantage of SIFT: providing fast, cheap guidance for tree search before committing to expensive downstream evaluations.

\vspace{-2mm}
\section{Discussion and Conclusion}
\vspace{-2mm}
We introduced \textbf{SIFT}, a lightweight framework for self-improving coding agents via fast tree search. The central idea is that not every candidate self-modification needs to be evaluated with an expensive downstream benchmark before it can provide useful search signal. Instead, SIFT uses pairwise LLM-as-a-judge comparisons between candidate agents, aggregates the resulting preferences with a regularized Bradley--Terry model, and uses the induced ranks to guide parent selection and evaluation priority inside a disaggregated tree-search pipeline.

Our results suggest that cheap intermediate preference signals can substantially improve the efficiency of recursive self-improvement. On Polyglot, SIFT finds agents that outperform prior self-evolution frameworks while using substantially less search compute. In particular, the best SIFT variants achieve strong full-benchmark performance with both \texttt{o3-mini} and \texttt{Qwen3-Coder-30B} coding backbones, transferring across LLMs. 

These findings support a broader point: in self-improvement, the main bottleneck is not only the generation of candidate modifications, but the cost of deciding which modifications deserve further investment. Full benchmark evaluation remains a reliable signal, but it is too expensive to apply uniformly to every node in a growing search tree. SIFT addresses this by separating cheap ranking from expensive verification. The judge provides fast, noisy, but useful relative preferences; the Bradley--Terry model turns those sparse comparisons into a global ranking; and the disaggregated pipeline allows expansion to continue while full evaluations are still pending. This combination preserves the benefits of empirical benchmark feedback while greatly reducing the amount of evaluation compute required to explore promising regions of the agent-design space.

An important practical observation is that the form of the judge input matters. Our ablations show that full-file pairwise comparisons produce a substantially stronger rank signal than judging only chains of diffs. This suggests that judges are more effective when they can reason about the final behavior of an agent implementation directly, rather than reconstructing that behavior from incremental patches. More generally, the success of SIFT depends on the judge being good enough to distinguish promising modifications from clearly harmful ones, but not necessarily perfectly calibrated: because SIFT uses ranks rather than raw scores, it only requires the judge to provide a useful ordering signal.

\paragraph{Limitations.}
SIFT still depends on downstream benchmark evaluations to ultimately validate improvements. The judge model used in our strongest runs is also stronger than the coding backbone, which makes the setup highly practical but not a pure form of self-judged improvement. Future work should investigate whether smaller or open-weight judges can provide comparable ranking signal, and whether the coding model itself can be used as an effective judge under appropriate prompting or ensembling.

The Bradley--Terry aggregation also assumes that pairwise judge outcomes can be explained by a single latent strength per node. This is a useful approximation for search, but it may hide task-specific tradeoffs: a patch that improves long-horizon debugging may hurt simple syntax-repair tasks, while another may have the opposite profile. A single scalar ranking cannot represent these dimensions. Extending SIFT to maintain multiple skill-specific rankings, to expand the strength scalar into a multi-dimensional vector, or to condition judge comparisons on failure categories, may further improve search efficiency.

Finally, recursive self-improvement introduces safety and evaluation-integrity risks. Candidate patches may learn to exploit weaknesses in the harness, relax constraints, or optimize for artifacts of the evaluation process rather than genuine coding ability. In our experiments, we mitigate this by sandboxing execution, restricting writable files, and rejecting patches that modify benchmark or harness code. These safeguards are essential for any self-improvement system that can edit its own implementation.

\paragraph{Future work.}
The pairwise-judge signal in SIFT is currently used mainly for parent sampling and evaluation prioritization. A natural extension is to use BT strength as an explicit evaluation gate, so that only candidates whose judge rank clears a dynamic threshold receive full benchmark evaluation. Another direction is archive management: nodes that are consistently dominated in both judge rank and measured accuracy could be pruned, allowing larger searches under fixed memory and compute budgets.

More powerful judges could also be made agentic. Instead of comparing two patches only from code, a judge could run lightweight targeted probes, inspect traces, or synthesize small adversarial tasks before issuing a pairwise preference. This would preserve SIFT's relative-comparison interface while giving the judge access to richer behavioral evidence. Such agentic judging may close the gap between cheap static assessment and expensive full benchmark evaluation, but comes at a tradeoff of higher number of LLM API calls.

Our proposed alternative to full benchmark evaluation using LLM-as-a-judge model to improve the signal-to-cost ratio makes use of ranking each candidate node with a scalar strength obtained by pairwise comparison. Another natural extension is to ask the question whether we can similarly rank downstream tasks in terms of difficulty, such that we sample the tasks with the highest signal-to-cost ratio. Existing works in Item Response Theory~\citep{lord1968statistical, maiapolo2024tinybenchmarks} study whether we can also model the difficulty of the task $i$ as a scalar parameter $\beta_i$ such that the probability of judge $j$ solving task $i$ is $p_{j,i} = \sigma(\theta_j - \beta_i)$. Given sufficient number of task evaluations, this can potentially achieve a good estimate on the difficulty of tasks, from which we can sample tasks maximizing for Fisher information.

\paragraph{Conclusion.}
SIFT shows that recursive self-improvement can be made substantially more sample- and cost-efficient by inserting a cheap preference-based ranking layer between candidate generation and full benchmark evaluation. Pairwise LLM judgments, regularized Bradley--Terry aggregation, and disaggregated tree search together provide a simple mechanism for exploring agent modifications without evaluating every patch exhaustively. The resulting agents improve under strict resource budgets, suggesting that fast preference-guided search is a promising direction for scalable self-improving coding systems.

\paragraph{Safety and Ethics.}
Recursive self-improvement amplifies standard agentic risks. In Polyglot-50 runs we observed the diagnosis agent occasionally propose patches that relaxed the evaluation harness itself (e.g.\ extending timeouts or retry counts); we block these by restricting writable files and rejecting patches that touch harness code. No such patches made it into the final experiments, but the pattern reinforces that sandboxing and an explicit allow-list of editable files are minimum hygiene for this class of experiments.

\paragraph{Acknowledgements}
The authors would like to thank Yoon Kim, Wenyi Wang and Rujikorn Charakorn, Qi Sun and other collaborators at MIT and Sakana AI for fruitful discussions and inputs throughout the project. Xinghong Fu was supported by the MIT UROP grant. Part of this work was completed during Xinghong's internship at Sakana AI.


\bibliography{neurips_bib}
\bibliographystyle{plainnat}

\clearpage
\appendix
\section{SWE-Bench experiments}
\label{appendix:swe_evals}
\suppressfloats[t]

\begin{wrapfigure}{r}{0.42\textwidth}
\vspace{-6mm}
\centering
\includegraphics[width=\linewidth]{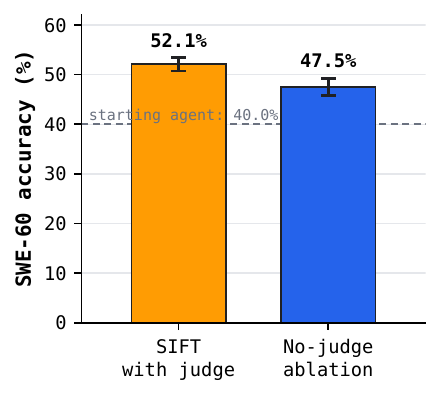}
\caption{Aggregate repeated-evaluation comparison. Bars show pooled mean accuracy and whiskers the standard error.}
\label{fig:swe60_pooled_sharpening}
\vspace{-6mm}
\end{wrapfigure}
We also evaluate SIFT on SWE-60, the 60-task subset of SWE-bench Verified~\citep{jimenez2024swebench} used in prior self-improvement research~\citep{hgm}. 
All SWE-60 experiments use \texttt{gpt-5-mini} as the coding model, \texttt{gpt-5} for diagnosis and self-improvement, and \texttt{gpt-5.4-high} as the judge model. 

We initialise our experiments from the same base coding agent as HGM, which scores 40.0\%, facilitating compariso against HGM, and limit the search budget to 800 evaluations.
We isolate the contribution of the judge by running an ablation of the judge.
In SIFT, pairwise LLM judgements are used to prioritise which self-improved agents receive further search and full evaluations.
In the no-judge ablation, this intermediate ranking signal is removed, so search decisions are based on SWE-60 accuracy and exploration alone.

\begin{table}
\centering
\small
\caption{SWE-60 judge/no-judge tree-search comparison. ``Search eval'' is the score observed during tree search. For each setting we rerun the two agents that the setting itself ranks best at the end of search: for SIFT, the judge's top two nodes by final Bradley--Terry strength; for the no-judge ablation, the two nodes with the highest measured search accuracy (node 1 ties two other nodes at 48.3\%). Raw scores list the search evaluation followed by three repeat evaluations; means include all four scores and parentheses give standard deviations in percentage points. The starting agent has a single complete full-60 evaluation, shared by all searches as their root. Node numbers are local to each run and do not indicate matched pairs.}
\label{tab:swe60_judge_nojudge}
\begin{tabularx}{\linewidth}{llccX}
\toprule
Setting & Agent & Search eval & Repeated mean & Raw full-60 scores \\
\midrule
Initial agent & Starting point & 40.0\% & -- & 40.0 \\
SIFT with judge & Judge rank 1 (node 16) & 53.3\% & 50.4\% (3.0) & 53.3, 46.7, 48.3, 53.3 \\
SIFT with judge & Judge rank 2 (node 12) & 55.0\% & 53.8\% (3.2) & 55.0, 48.3, 56.7, 55.0 \\
No-judge ablation & Accuracy rank 1 (node 11) & 51.7\% & 44.6\% (4.8) & 51.7, 38.3, 45.0, 43.3 \\
No-judge ablation & Accuracy rank 2 (node 1) & 48.3\% & 50.4\% (1.4) & 48.3, 50.0, 51.7, 51.7 \\
\bottomrule
\end{tabularx}
\end{table}

\begin{figure}
\centering
\includegraphics[width=\linewidth]{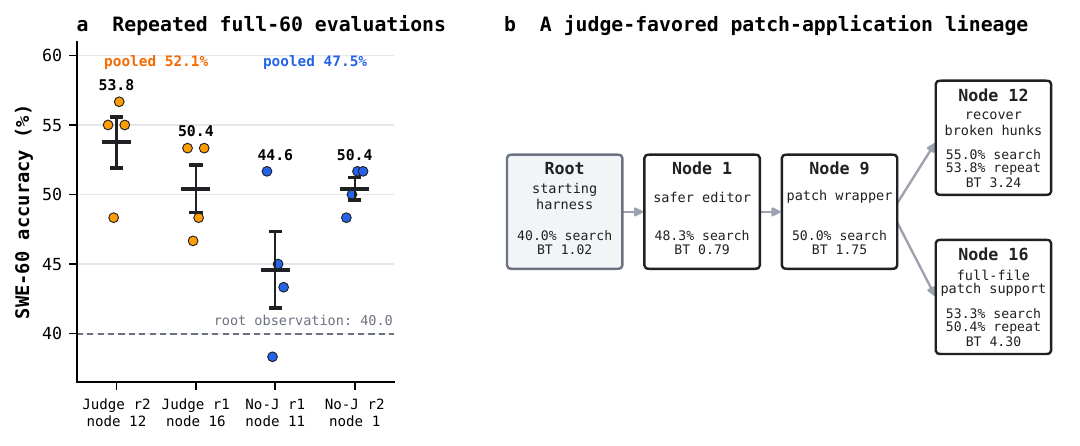}
\caption{SWE-60 repeatability and a qualitative judged lineage. \textbf{Left:} individual full-60 evaluations of the four selected agents; horizontal markers show means and error bars show standard errors, and the dashed line marks the single 40.0\% root observation. \textbf{Right:} one lineage favored during judge-guided search. Stored search-time BT strengths are shown alongside search accuracy; the two selected descendants additionally show their four-evaluation means.}
\label{fig:swe60_sharpening}
\end{figure}

For the judge and no-judge runs respectively Table~\ref{tab:swe60_judge_nojudge} show the two agents that ranked best at the end of the search.
Because individual SWE-60 evaluations are noisy, we report both the score observed during search and the full evaluations of the selected agents on the 60 tasks. 
We see that the judge-based selections lead the no-judge selections in both the search evaluation and under repetition.
We summarize this further in the aggregate presented in Figure~\ref{fig:swe60_pooled_sharpening}.
Together this suggests a sharpening effect from the judge.

\paragraph{SWE-60 judge preferences.}
In the right panel of Figure~\ref{fig:swe60_sharpening} we give a concrete example of the leading lineage chosen by the judge. 
In this lineage, support for LLM-style patch wrappers is first added to the agent's \texttt{git\_patch} tool. 
The two children then refine this.
Node 12 adds a guarded recovery path for malformed patches, while Node 16 adds support for accepting full-file and add-file wrapper blocks.
In the pairwise explanations, the judge prefers Node 16 over Node 12 due to concerns that Node 12 may allow for silently placing an edit incorrectly.
We generally observe that the judge has a bias toward conservative changes rather than towards additional machinery, which leads to the imperfect ordering.
This may be rectifiable through tuning the judge prompt.

\section{Technical appendices and supplementary material}


\subsection{Hyperparameters for Bradley–Terry model}
In practice, we use $\alpha = 1, \beta = 1, \eta  = 1$. 
We observe that with $\eta = 0$, the lack of sufficient exploration within the tree search process leads to a reduced performance, reaching 30.1\% on Polyglot (as opposed to 35.1\% with $\eta = 1$) using \texttt{o3-mini} as the base coding model. 
Further tuning of these hyperparameters may yield a better exploration-exploitation tradeoff and improve final agent performance.
The same hyperparameters are used across all reported tree search runs.

\subsection{Full-file judge prompt}
\label{app:full_file_judge_prompt}

SIFT uses a \texttt{full\_files} judge prompt mode in which the judge sees each candidate agent's post-patch runtime implementation rather than only a chain of diffs. In this mode, each judge call contains a system prompt and a user message. The user message is formatted with two slots: \slot{num\_candidates}, the number of candidates being compared, and \slot{candidates\_section}, the rendered view of each candidate's differing runtime files. We use the \texttt{default} full-file prompt variant.

\paragraph{System prompt.}
\begin{promptbox}{\texttt{FULL\_FILES\_SYSTEM\_PROMPT}}
\begin{Verbatim}[breaklines=true,breakanywhere=true,fontsize=\small]
You are an expert code reviewer evaluating multiple candidate modifications of an AI coding agent. For each candidate you will be shown its FULL implementation files (already modified). Compare them directly.

Your task is to RANK the candidates from most to least likely to improve the agent's performance.

## IMPORTANT: Rank the candidates on how likely they are to improve performance, but also check for correctness. You can assume that imports to files that you cannot see will work fine.

You MUST respond in the following JSON format and nothing else.
Write the `reasoning` field FIRST so your critique drives the ranking, not the
other way around — once the ranking tokens are emitted you cannot revise them:
{"reasoning": "<brief explanation of your ranking>", "ranking": [<list of candidate ids from best to worst>]}
\end{Verbatim}
\end{promptbox}

\paragraph{User template.}
\begin{promptbox}{\texttt{FULL\_FILES\_USER\_TEMPLATE}}
\begin{Verbatim}[breaklines=true,breakanywhere=true,fontsize=\small]
Below are {num_candidates} candidate modifications of the AI coding agent. Each candidate's complete implementation files are shown. Rank them from most to least likely to improve agent performance.

{candidates_section}

Respond with JSON only, `reasoning` field first then `ranking`: {"reasoning": "<str>", "ranking": [<candidate ids best to worst>]}
\end{Verbatim}
\end{promptbox}

\paragraph{Candidate rendering.}
The \slot{candidates\_section} field is produced by collecting the patched \texttt{coding\_agent.py} file together with all Python files under \texttt{tools/} and \texttt{utils/}, excluding cache files. Files that are byte-identical across all compared candidates are omitted, so the judge sees only the runtime files that differ. Files missing from a candidate are marked explicitly.

\begin{promptbox}{Rendered \texttt{candidates\_section} schematic}
\begin{Verbatim}[breaklines=true,breakanywhere=true,fontsize=\small]
Files that are byte-identical across all candidates have been omitted; only files that differ between candidates are shown below.

## Candidate <id_1>

### <relative/path/to/file_a.py>
```python
<full file contents, truncated at 30000 chars>
```

### <relative/path/to/file_b.py>
```python
<full file contents>
```

## Candidate <id_2>

### <relative/path/to/file_a.py>
```python
<full file contents>
```

### <relative/path/to/file_b.py>
(file absent in this candidate)
\end{Verbatim}
\end{promptbox}

\subsection{Variable tracking in SIFT}
\begin{table}[ht]
\centering
\caption{State variables maintained by SIFT during tree search.}
\label{tab:sift_state}
\begin{tabular}{lc}
\toprule
\textbf{Variable} & \textbf{Definition} \\
\midrule
$W_{ij}$ & Number of judge wins of node $i$ over node $j$ \\
$b_i$ & BT strength of node $i$ from pairwise judge outcomes \\
$a_i$ & Measured accuracy of node $i$ on the evaluation subset\\
$r_b(i)$ & Rank of node $i$ by BT strength, with best rank $0$\\
$r_a(i)$ & Rank of node $i$ by accuracy, with best rank $0$  \\
$v_i$ & Number of times node $i$ has been selected as a parent \\
\bottomrule
\end{tabular}
\end{table}

\subsection{Theoretical properties of BT-guided judging}

The role of the BT judge in SIFT is not to certify that a candidate patch is
truly better on the downstream benchmark. Instead, the judge provides a cheap
intermediate signal for deciding which parts of the self-improvement tree should
be explored before committing to expensive benchmark evaluation. Benchmark
accuracy remains an important validation signal; the judge is used only as a noisy
ranking oracle. The Bradley--Terry aggregation step gives a principled way to
convert sparse, potentially inconsistent pairwise preferences into a global
ordering over archive nodes.

We formalize three properties of this mechanism. First, the regularized BT
objective used by SIFT is well posed even when the comparison graph is sparse.
Second, in an idealized Bradley--Terry (BT) comparison model, the fitted
BT scores recover the top-$K$ latent judge frontier from sparse pairwise
comparisons. Third, if the recovered ranking is informative, the exponential
rank-based parent sampler substantially increases the probability of expanding
frontier nodes relative to uniform sampling.

\paragraph{Bradley--Terry fit details.}
The Bradley--Terry model can be viewed as logistic regression over pairwise differences in latent skill. Let $s_i = \log \theta_i$ denote the log-strength of node $i$. Then
\[
    P(i \succ j)
    =
    \frac{\theta_i}{\theta_i + \theta_j}
    =
    \frac{\exp(s_i)}{\exp(s_i)+\exp(s_j)}
    =
    \sigma(s_i - s_j),
\]
where $\sigma$ is the logistic sigmoid. Therefore, each LLM-judge preference contributes evidence about the difference $s_i-s_j$, not about the absolute value of either score. This is why BT is a natural fit for our judge interface: the judge only has to decide which of two candidate patches is better, while the aggregation step recovers a global ordering over all nodes.

Given the win-count matrix $W$, the unregularized log-likelihood is
\[
    \ell(\theta)
    =
    \sum_{i \neq j}
    W_{ij}
    \log
    \frac{\theta_i}{\theta_i+\theta_j}.
\]
We use a smoothed version by adding a symmetric pseudo-count $\lambda$ to every ordered pair:
\[
    \ell_\lambda(\theta)
    =
    \sum_{i \neq j}
    (W_{ij}+\lambda)
    \log
    \frac{\theta_i}{\theta_i+\theta_j}.
\]
Equivalently, this adds $\lambda$ fictitious wins for $i$ over $j$ and $\lambda$ fictitious wins for $j$ over $i$ for every unordered pair. This prevents degenerate estimates when the comparison graph is sparse or when a node has not yet lost. Such sparsity is typical in SIFT because each newly generated node is compared against at most $K$ incumbent nodes rather than against the entire archive.

We solve the regularized BT likelihood using the standard minorization--maximization (MM) fixed-point update:
\[
\gamma_i^{(t+1)}
=
\frac{
    \sum_j (W_{ij}+\lambda)
}{
    \sum_j
    \frac{
        W_{ij}+W_{ji}+2\lambda
    }{
        \gamma_i^{(t)}+\gamma_j^{(t)}
    }
}.
\]
After every update, we normalize the scores so that $\sum_i \gamma_i = n$. This normalization is necessary because BT strengths are only identifiable up to a global multiplicative constant: multiplying all strengths by the same positive scalar leaves every pairwise probability unchanged. We iterate until the maximum absolute change in normalized strengths is below $10^{-6}$ or until 100 iterations have elapsed.

In our implementation, self-comparisons are skipped. The root node is assigned $\gamma_0=1$ for initialization, but the fitted ranking used by SIFT is determined by the accumulated pairwise judge outcomes among archive nodes. Once the strengths are fitted, we sort nodes by $\gamma_i$ to obtain the BT rank $r_{\text{judge}}(i)$. SIFT uses this rank, rather than the raw $\gamma_i$, in parent sampling. This rank-based use is intentional: it preserves the ordering information supplied by the judge while avoiding over-sensitivity to the numerical magnitude of BT scores, which can vary with the number and topology of comparisons.

\paragraph{Regularized BT objective.}
The win count matrix $W_{ij}$ denotes the number of times node $i$ is preferred to node $j$ by the
judge, with $W_{ii}=0$. 
We can express the parameters $\theta_i$ in terms of $s_i=\log \theta_i$, the log-strength of node $i$.
Thus SIFT fits a smoothed Bradley--Terry model by maximizing
\[
    \ell_\lambda(s)
    =
    \sum_{i\neq j}
    (W_{ij}+\lambda)
    \left[
        s_i-\log\!\left(e^{s_i}+e^{s_j}\right)
    \right],
\]
Since BT scores are identifiable
only up to an additive shift in $s$, we optimize over the normalized subspace
\[
    \mathcal{H}
    =
    \left\{
        s\in\mathbb{R}^n:
        \sum_{i=1}^n s_i=0
    \right\}.
\]

\begin{theorem}[Well-posedness of regularized BT]
\label{thm:bt_well_posed}
For any win-count matrix $W$ with nonnegative entries and any $\lambda>0$,
the regularized BT log-likelihood $\ell_\lambda$ has a unique maximizer on
$\mathcal{H}$.
\end{theorem}

\begin{proof}
We first show strict concavity on $\mathcal{H}$. For each unordered pair
$\{i,j\}$, define
\[
    m_{ij}=W_{ij}+W_{ji}+2\lambda .
\]
Because $\lambda>0$, we have $m_{ij}>0$ for every $i\neq j$. The contribution of
the pair $\{i,j\}$ to the Hessian of $\ell_\lambda$ is
\[
    -m_{ij}\,
    \sigma(s_i-s_j)\sigma(s_j-s_i)
    (e_i-e_j)(e_i-e_j)^\top,
\]
where $\sigma(x)=1/(1+e^{-x})$ and $e_i$ is the standard basis vector with a 1 in the $i$-th coordinate and 0 everywhere else. Hence, for any vector $x\in\mathbb{R}^n$,
\[
    x^\top \nabla^2 \ell_\lambda(s) x
    =
    -
    \sum_{i<j}
    m_{ij}\,
    \sigma(s_i-s_j)\sigma(s_j-s_i)
    (x_i-x_j)^2 .
\]
The coefficients are strictly positive for all $i<j$. Therefore the quadratic
form is zero only when $x_i=x_j$ for all $i,j$, i.e. when $x$ is proportional to
the all-ones vector. On the normalized subspace $\mathcal{H}$, the only such
direction is $x=0$. Thus $\ell_\lambda$ is strictly concave on $\mathcal{H}$.

It remains to show that a maximizer exists. Because the pseudo-count adds
positive mass in both directions for every pair, the objective tends to
$-\infty$ whenever the range $\max_i s_i-\min_i s_i$ tends to infinity. On
$\mathcal{H}$, $\|s\|_2\to\infty$ implies that this range tends to infinity.
Therefore the upper level sets of $\ell_\lambda$ are bounded in $\mathcal{H}$.
Since $\ell_\lambda$ is continuous, it attains its maximum on a compact upper
level set. Strict concavity then implies that this maximizer is unique.
\end{proof}

This theorem justifies the pseudo-count used by SIFT. Each new node is compared
against only a small number of incumbents, so without regularization a node with
a few wins and no losses could receive an unstable or unbounded score. The
pseudo-count prevents this degeneracy and makes BT scores comparable across
nodes with different comparison counts.

\paragraph{Idealized top-$K$ recovery.}
We next state an idealized statistical guarantee for BT aggregation. This result
does not exactly model the adaptive comparison graph used by SIFT. Instead, it
isolates the statistical benefit of BT aggregation in the standard random-design
setting studied in the pairwise-ranking literature.

\begin{assumption}[Random-design BT judge model]
\label{assump:bt}
There are $n$ fixed candidate agents, each with a latent judge utility
$s_i^\star\in\mathbb{R}$. For identifiability, assume
$\sum_i s_i^\star=0$ and $\|s^\star\|_\infty\le B$ for a constant $B$.
For each unordered pair $(i,j)$, the comparison graph includes edge $(i,j)$
independently with probability $p$. Each observed edge is compared $L$
independent times. Every comparison is generated according to the
Bradley--Terry--Luce model
\[
    \Pr(i\succ j)
    =
    \sigma(s_i^\star-s_j^\star).
\]
\end{assumption}

Let $s_{(1)}^\star\ge s_{(2)}^\star\ge\cdots\ge s_{(n)}^\star$ denote the sorted
latent judge utilities, and let $S_K^\star$ be the set of indices corresponding
to the $K$ largest entries. Define the top-$K$ separation gap
\[
    \Delta_K
    =
    s_{(K)}^\star-s_{(K+1)}^\star .
\]

\begin{theorem}[Idealized top-$K$ recovery]
\label{thm:topk_recovery}
Under Assumption~\ref{assump:bt}, suppose the comparison graph is sufficiently
connected and the dynamic range $B$ is bounded. Let $\hat s$ be the regularized
BT maximum likelihood estimator, normalized so that $\sum_i \hat s_i=0$.
Existing top-$K$ ranking results for the BT model imply that, with high
probability,
\[
    \|\hat s-s^\star\|_\infty
    \le
    C_B
    \sqrt{\frac{\log n}{npL}},
\]
where $C_B$ depends only on the dynamic-range bound $B$. Consequently, if
\[
    \Delta_K
    >
    2C_B
    \sqrt{\frac{\log n}{npL}},
\]
then sorting nodes by $\hat s_i$ exactly recovers the true top-$K$ set
$S_K^\star$.
\end{theorem}

\begin{proof}
The entrywise error bound follows from standard top-$K$ ranking results for the
BT model under random non-adaptive sampling; see, e.g.,
\citet{chen2015spectralmle} and \citet{chen2019spectralregularizedmle}. We use
this bound to prove exact top-$K$ recovery.

Let
\[
    \varepsilon
    =
    C_B
    \sqrt{\frac{\log n}{npL}},
\]
and suppose that $\|\hat s-s^\star\|_\infty\le\varepsilon$. For any
$i\in S_K^\star$ and any $j\notin S_K^\star$, we have
\[
    s_i^\star-s_j^\star
    \ge
    \Delta_K .
\]
Therefore,
\[
\begin{aligned}
    \hat s_i-\hat s_j
    &=
    (s_i^\star-s_j^\star)
    +
    (\hat s_i-s_i^\star)
    -
    (\hat s_j-s_j^\star) \\
    &\ge
    \Delta_K - 2\varepsilon .
\end{aligned}
\]
If $\Delta_K>2\varepsilon$, then $\hat s_i>\hat s_j$ for every
$i\in S_K^\star$ and every $j\notin S_K^\star$. Hence all true top-$K$ items are
ranked above all non-top-$K$ items by $\hat s$, so sorting by the BT estimate
exactly recovers $S_K^\star$.
\end{proof}

Theorem~\ref{thm:topk_recovery} should be read as an idealized guarantee rather
than a direct finite-sample theorem for SIFT. SIFT uses an adaptive,
frontier-biased comparison graph: each new candidate is compared against a small
set of strong incumbents rather than against an Erd\H{o}s--R\'enyi sample of all
possible pairs. Nevertheless, the theorem explains why BT aggregation is a
natural design choice. If the LLM judge behaves approximately as a noisy
preference oracle whose latent utilities are positively correlated with
downstream agent quality, then sparse pairwise comparisons can recover a useful
frontier ranking without exhaustively comparing all $\binom{n}{2}$ node pairs.

\paragraph{Rank amplification by parent sampling.}
SIFT uses BT scores only through their induced ranks. The following simple
calculation shows why even a moderately informative ranking can substantially
bias search toward frontier nodes.

\begin{proposition}[Probability mass on the top ranks]
\label{prop:rank_amplification}
Suppose the archive contains $n$ nodes ranked $r=0,1,\ldots,n-1$, where
$r=0$ is best. If parent sampling uses
\[
    P(r)
    =
    \frac{\exp(-\alpha r)}
    {\sum_{\ell=0}^{n-1}\exp(-\alpha \ell)}
\]
with $\alpha>0$, then the total probability assigned to the top $M$ nodes is
\[
    P(\mathrm{top}\ M)
    =
    \frac{1-e^{-\alpha M}}{1-e^{-\alpha n}}.
\]
Uniform sampling assigns probability $M/n$ to the same set.
\end{proposition}

\begin{proof}
The result follows by summing a finite geometric series:
\[
    P(\mathrm{top}\ M)
    =
    \frac{\sum_{r=0}^{M-1} e^{-\alpha r}}
    {\sum_{r=0}^{n-1} e^{-\alpha r}}
    =
    \frac{(1-e^{-\alpha M})/(1-e^{-\alpha})}
    {(1-e^{-\alpha n})/(1-e^{-\alpha})}
    =
    \frac{1-e^{-\alpha M}}{1-e^{-\alpha n}}.
\]
\end{proof}

Thus, once BT aggregation produces an informative ranking, exponential
rank-based sampling converts that ranking into search pressure toward promising
nodes. Compared with uniform sampling, the multiplicative increase in probability
mass on the top $M$ nodes is
\[
    \frac{P(\mathrm{top}\ M)}{M/n}
    =
    \frac{n}{M}
    \cdot
    \frac{1-e^{-\alpha M}}{1-e^{-\alpha n}}.
\]
For small $M$ and moderate $\alpha$, this factor can be large.

\paragraph{Visit-count penalty.}
The previous proposition ignores the visit-count term used by SIFT. For a fixed
archive, define
\[
    q_i
    =
    \exp\!\left(
        -\alpha r_{b}(i)
        -\beta r_{a}(i)
    \right)
    >0 .
\]
SIFT samples node $i$ at time $t$ with probability
\[
    P_t(i)
    =
    \frac{
        q_i(1+v_i(t))^{-\eta}
    }{
        \sum_j q_j(1+v_j(t))^{-\eta}
    },
\]
where $v_i(t)$ is the number of previous times node $i$ has been selected as a
parent.

\begin{proposition}[No permanent starvation in a fixed archive]
\label{prop:no_starvation}
For any fixed finite archive and any $\eta\ge 0$, every node is sampled
infinitely often almost surely.
\end{proposition}

\begin{proof}
Assume for contradiction that some node $i$ is sampled only finitely many times.
Then $v_i(t)$ is eventually constant, so
$q_i(1+v_i(t))^{-\eta}$ is eventually bounded below by a positive constant.
For every node $j$, we have $(1+v_j(t))^{-\eta}\le 1$, so the denominator is at
most $\sum_j q_j$. Hence, after some finite time, the conditional probability of
sampling node $i$ is bounded below by a positive constant $\rho>0$. The
probability of never sampling $i$ again is then at most
\[
    \lim_{T\to\infty}(1-\rho)^T=0,
\]
contradicting the assumption. Therefore every node is sampled infinitely often
almost surely.
\end{proof}

In SIFT the archive is not fixed, so Proposition~\ref{prop:no_starvation} is not
a full exploration guarantee for the growing tree. It nevertheless captures the
purpose of the visit-count penalty: high-ranked nodes are favored, but repeated
selection gradually reduces their sampling weight, preventing the search from
collapsing permanently onto a single lineage.

\paragraph{Full algorithm}
In Algorithm ~\ref{alg:sift} we list the full algorithm used to execute the tree search in SIFT.
\begin{algorithm}[hbt!]
\caption{SIFT: Self-Improvement via Fast Tree Search}
\label{alg:sift}
\begin{algorithmic}[1]
\STATE \textbf{Input:} base agent $A_0$, judge $J$, evaluation subset $\mathcal{B}$, max nodes $N$, targets per node $K$, weights $(\alpha,\beta,\eta)$, BT prior $\lambda$
\STATE Evaluate $A_0$ on $\mathcal{B}$; initialize archive $\mathcal{T}\leftarrow\{A_0\}$, win matrix $W\leftarrow 0$, visit counts $v_i\leftarrow 0$
\STATE Set initial accuracy $a_0$ from evaluation and initial BT strength $\theta_0\leftarrow 1$
\WHILE{$|\mathcal{T}| < N$}
    \STATE \textbf{Collect completed expansions.}
    \FOR{each finished expansion producing child node $c$ with parent $p$}
        \STATE Run easy evaluation on $c$
        \IF{$c$ fails the easy evaluation}
            \STATE Mark $c$ as failed and discard it from full evaluation
            \STATE \textbf{continue}
        \ENDIF
        \STATE Add $c$ to archive: $\mathcal{T}\leftarrow \mathcal{T}\cup\{c\}$
        \STATE Temporarily set $a_c\leftarrow a_p$ until full evaluation completes
        \STATE Select up to $K$ comparison targets from $\mathcal{T}\setminus\{c\}$ using $\alpha r_{b}+\beta r_{a}$
        \FOR{each comparison target $t$}
            \STATE Query judge $J(c,t)$ and update $W_{c,t}$ or $W_{t,c}$
        \ENDFOR
        \STATE Refit regularized BT:
        \[
        \{\theta_i\}_{i\in\mathcal{T}}
        \leftarrow
        \arg\max_{\theta_i>0}
        \sum_{i\neq j}
        (W_{ij}+\lambda)
        \log
        \frac{\theta_i}{\theta_i+\theta_j}
        \]
        \STATE Enqueue $c$ for full evaluation with priority
        \[
        \alpha r_{b}(c)+\beta r_{a}(c)
        \]
    \ENDFOR

    \STATE \textbf{Collect completed full evaluations.}
    \FOR{each finished full evaluation of node $c$}
        \STATE Update measured accuracy $a_c$
    \ENDFOR

    \STATE \textbf{Spawn new expansions.}
    \WHILE{an expansion slot is available and $|\mathcal{T}| < N$}
        \STATE Sample parent $p\in\mathcal{T}$ with
        \[
        P(p)
        \propto
        \exp\!\left(
        -\alpha r_{b}(p)
        -\beta r_{a}(p)
        -\eta\log(1+v_p)
        \right)
        \]
        \STATE Launch asynchronous \textsc{Diagnose+SelfImprove} from parent $p$
        \STATE Increment $v_p\leftarrow v_p+1$
    \ENDWHILE

    \STATE Launch full evaluations from the priority queue while evaluator slots are available
    \STATE Checkpoint current archive, win matrix, evaluations, and pending jobs
\ENDWHILE
\STATE \textbf{Return} the node in $\mathcal{T}$ with highest measured accuracy.
\end{algorithmic}
\end{algorithm}

\subsection{Discovered agent examples}
\label{app:interesting_agent}

To better understand what kinds of agent modifications SIFT discovers, we inspect one example high-performing lineage. Starting from the initial coding agent of \texttt{o3-mini}, SIFT first discovers node 9, which reaches $0.44$ with a very small intervention: a one-line prompt edit that instructs the agent to run tests after editing, plus a new \texttt{test\_runner} tool. A later child, node 24, inherits node 9's changes and adds a substantially larger \texttt{test\_analyzer} mechanism, but its realized accuracy remains $0.44$ and its BT score is lower than node 9's. This example illustrates that SIFT can identify simple, active-path changes whose benefit is not improved by adding extra machinery.

\begin{center}
\small
\begin{tabular}{rrrrl}
\toprule
Node & Parent & Accuracy & BT score & Cumulative change vs.\ baseline \\
\midrule
9  & root (0) & 0.44 & 2.370 & +1 instruction line, +1 new tool \\
24 & 9        & 0.44 & 1.903 & node 9's changes + agent edit + 1 new tool \\
\bottomrule
\end{tabular}
\end{center}

\paragraph{Node 9.}
Node 9 changes the task instruction inside \texttt{forward()} by appending a short instruction that asks the agent to run a new \texttt{test\_runner} tool after editing and use structured failures for iterative repair.

\begin{codebox}{Cumulative diff: \texttt{coding\_agent.py} for node 9}
\begin{lstlisting}[style=diffstyle]
@@ -183,7 +183,7 @@

 Your task is to make changes to the files in the {self.git_dir} directory to address the <problem_description>. I have already taken care of the required dependencies.
 """
-        instruction = f"{task}\n\nPlease analyze the problem description carefully. Then make edits to the code files to complete the instruction."
+        instruction = f"{task}\n\nPlease analyze the problem description carefully. Then make edits to the code files to complete the instruction.\n\nAfter making edits run the 'test_runner' tool to execute the project's test suite and obtain structured failure information. Use the failures to make small, targeted edits and re-run tests iteratively until the test suite passes or a safe iteration limit is reached. Include test-driven steps such as implementing or refining tests if necessary."
         chat_history = chat_with_agent(instruction, model=self.code_model, msg_history=[], logging=safe_log)
         chat_history_str = str(chat_history)
\end{lstlisting}
\end{codebox}

The new \texttt{test\_runner} tool auto-detects common project test frameworks, including \texttt{pytest}, \texttt{cargo}, \texttt{go test}, \texttt{npm}, Gradle, and Maven, and returns structured pass/fail output with parsed failures. The agent sees the following tool schema.

\begin{codebox}{\texttt{tools/test\_runner.py}: exposed tool schema}
\begin{lstlisting}[style=pystyle]
def tool_info():
    return {
        "name": "test_runner",
        "description": "Run repository tests and return structured "
                       "results and failure parsing across common "
                       "languages (pytest, cargo, go test, npm, "
                       "gradle/maven).",
        "input_schema": {
            "type": "object",
            "properties": {
                "command": {"type": "string",
                            "description": "Test command to run "
                            "(optional). If omitted the tool will "
                            "attempt to auto-detect an appropriate "
                            "test command based on repo files."},
                "timeout": {"type": "number",
                            "description": "Timeout in seconds for "
                            "the test command."},
                "cwd":     {"type": "string",
                            "description": "Working directory in "
                            "which to run the tests (optional)."},
            },
            "required": []
        }
    }
\end{lstlisting}
\end{codebox}

\paragraph{Node 24.}
Node 24 keeps node 9's \texttt{test\_runner} intervention, then adds a second tool, \texttt{test\_analyzer}, and modifies the agent to call it before invoking the LLM. The analyzer attempts to infer public APIs from test files and optionally apply a scaffold patch. Despite this larger change, node 24 does not improve the measured accuracy over node 9. The subset accuracy of both are 44\% and the full Polyglot accuracies of node 9 and node 24 are 35.6\% and 33.8\% respectively.

\begin{codebox}{Additional cumulative diff: \texttt{coding\_agent.py} for node 24}
\begin{lstlisting}[style=diffstyle]
@@ -7,6 +7,7 @@

 from llm_withtools import CLAUDE_MODEL, OPENAI_MODEL, chat_with_agent
 from utils.git_utils import diff_versus_commit, reset_to_commit, apply_patch
+from tools import load_all_tools

@@ -183,10 +184,54 @@

-        instruction = f"{task}\n\nPlease analyze the problem description carefully. Then make edits to the code files to complete the instruction."
+        instruction = f"{task}\n\nPlease analyze the problem description carefully. Then make edits to the code files to complete the instruction.\n\nAfter making edits run the 'test_runner' tool to execute the project's test suite and obtain structured failure information. Use the failures to make small, targeted edits and re-run tests iteratively until the test suite passes or a safe iteration limit is reached. Include test-driven steps such as implementing or refining tests if necessary."
+
+        # Run test analyzer tool prior to invoking the LLM so the agent has API inference context
+        try:
+            all_tools = load_all_tools(logging=safe_log)
+            tools_dict = {t['info']['name']: t for t in all_tools}
+            analyzer = tools_dict.get('test_analyzer')
+            analysis_json = None
+            if analyzer:
+                test_paths = []
+                for root, _, files in os.walk(self.git_dir):
+                    for fn in files:
+                        if self.language == 'cpp' and (fn.endswith('_test.cpp') or fn.endswith('.cpp') and 'test' in fn.lower()):
+                            test_paths.append(os.path.join(root, fn))
+                        if self.language == 'python' and (fn.startswith('test_') or fn.endswith('_test.py')):
+                            test_paths.append(os.path.join(root, fn))
+                try:
+                    analyzer_result_str = analyzer['function'](
+                        paths=test_paths,
+                        language=self.language,
+                        generate_scaffold=True,
+                    )
+                    analysis_json = analyzer_result_str
+                    safe_log(f"[test_analyzer] result: {analysis_json}")
+                    try:
+                        parsed = json.loads(analysis_json)
+                        scaffold = parsed.get('suggested_scaffold')
+                        if scaffold:
+                            safe_log("[test_analyzer] Applying suggested scaffold patch")
+                            apply_patch(self.git_dir, scaffold)
+                    except Exception as e:
+                        safe_log(f"[test_analyzer] could not parse scaffold: {e}")
+                except Exception as e:
+                    safe_log(f"Error running test_analyzer: {e}")
+            else:
+                safe_log('test_analyzer tool not found')
+        except Exception as e:
+            safe_log(f'Error loading tools: {e}')
+
+        if analysis_json:
+            instruction = instruction + "\n\n# Inferred API from tests:\n" + str(analysis_json)

         chat_history = chat_with_agent(instruction, model=self.code_model, msg_history=[], logging=safe_log)
\end{lstlisting}
\end{codebox}

The exposed schema for the new analyzer tool is:

\begin{codebox}{\texttt{tools/test\_analyzer.py}: exposed tool schema}
\begin{lstlisting}[style=pystyle]
def tool_info():
    return {
        'name': 'test_analyzer',
        'description': 'Parse test files to infer expected public API '
                       'and optionally generate scaffolding patches. '
                       'Supports C++, Python, Java, JavaScript, Rust, Go '
                       'with heuristics and returns a JSON description '
                       'of required symbols and a patch suggestion.',
        'input_schema': {
            'type': 'object',
            'properties': {
                'paths':              {'type': 'array',
                                       'items': {'type': 'string'},
                                       'description': "List of file paths "
                                       "(absolute or repo-relative) to "
                                       "analyze."},
                'language':           {'type': 'string',
                                       'description': "Language to "
                                       "prioritize when parsing."},
                'generate_scaffold':  {'type': 'boolean',
                                       'description': 'Whether to return '
                                       'a suggested patch for scaffolding '
                                       'files.'},
            },
            'required': ['paths', 'language'],
        },
    }
\end{lstlisting}
\end{codebox}

The comparison suggests that the most useful discovered modification was not a broad rewrite, but a small active-path change that made the agent more test-driven. The child node adds more automation and more surface area, but does not improve realized accuracy. This is consistent with the role of the judge in SIFT: it can help prioritize changes that are likely to improve actual solve behavior while penalizing additional complexity that introduces runtime or regression risk.


\newpage

\end{document}